\pdfoutput=1
\documentclass[11pt,a4paper]{article}
\usepackage[hyperref]{acl2017}
\usepackage{times}
\usepackage{latexsym}

\usepackage{url}

\aclfinalcopy 
\def\aclpaperid{799} 

\usepackage{mathptmx,graphicx,amsmath,amssymb}

\usepackage{amsmath,amsthm}
\newtheorem{proposition}{Proposition}
\usepackage{tikz-dependency}
\usepackage{my-selfloop}
\renewcommand\backslash{\reflectbox{\rotatebox[origin=c]{180}{\texttt{/}}}}

\usepackage{array,ragged2e}\newcolumntype{P}[1]{>{\RaggedRight\hspace{0pt}}p{#1}}

\newcommand{\changed}[1]{\textcolor{purple}{#1}}
\renewcommand{\changed}[1]{{#1}}

\newcommand{\specialcell}[2][c]{%
  \begin{tabular}[#1]{@{}c@{}}#2\end{tabular}}

\begin{document}
\title{Generic Axiomatization of Families of Noncrossing Graphs in
  Dependency Parsing}

\author{Anssi Yli-Jyr\"a \\ University of Helsinki, Finland \\ {\tt anssi.yli-jyra@helsinki.fi}
  \And  Carlos G\'omez-Rodr\'\i{}guez \\ Universidade da Coru\~na, Spain \\
 {\tt carlos.gomez@udc.es} }
\date{}
\maketitle
\begin{abstract}
  We present a simple encoding for unlabeled noncrossing graphs and
  show how its latent counterpart helps us to represent several
  families of directed and undirected graphs used in syntactic and
  semantic parsing of natural language as context-free languages.  The
  families are separated purely on the basis of forbidden patterns in
  latent encoding, eliminating the need to differentiate the families of
  non-crossing graphs in inference algorithms: one algorithm works for all
  when the search space can be controlled in parser input.
\end{abstract}

\section{Introduction}

Dependency parsing has received wide attention in recent years, as accurate and efficient dependency parsers have appeared that are applicable to many languages. Traditionally, dependency parsers have produced syntactic analyses in tree form, including exact inference algorithms that search for maximum projective trees \cite{Eisner} and
maximum spanning trees \cite{McDonald} in weighted digraphs, as well as greedy and beam-search approaches that forgo exact search for extra efficiency \cite{Zhang2011rich}.

Recently, there has been growing interest in providing a richer analysis of natural language by going beyond trees. In semantic dependency parsing \cite{oepen-EtAl:2015:SemEval,KuhOe2016}, the desired syntactic representations can have in-degree greater than 1 (re-entrancy), suggesting the search for maximum acyclic subgraphs \cite{Schluter2014,Schluter2015}. As this inference task is intractable \cite{Guruswami-etal-2011}, noncrossing digraphs have been studied instead, e.g. by \citet{Kuhlmann-Jonsson-2015} who provide a $O(n^3)$ parser for maximum noncrossing acyclic
subgraphs.

\citet{DBLP:journals/ijfcs/Yli-Jyra05} studied how to
axiomatize dependency trees as a special case of noncrossing digraphs.  This gave rise to a new homomorphic representation of context-free languages that proves the classical Chomsky and Sch\"utzenberger
theorem using a quite different internal language.  In this language, the brackets indicate arcs in a dependency tree in a way that is reminiscent to encoding schemes used earlier by  \citet{Greibach} and \citet{DBLP:journals/coling/Oflazer03}.  Cubic-time parsing algorithms that are incidentally or intentionally applicable to this kind of homomorphic representations have been considered, e.g., by
\citet{Nederhof-2003}, \citet{Hulden2011}, and \citet{DBLP:conf/birthday/Yli-Jyra12}.

Extending these insights to arbitrary noncrossing digraphs, or to relevant families of them, is far from obvious. In this paper, \changed{we develop (1) a \emph{linear encoding} supporting general noncrossing digraphs, and (2) show that the encoded noncrossing digraphs form a \emph{context-free language}.   We then give it (3) two  \emph{homomorphic, nonderivative representations} and use the latent local features of the latter to characterize various families of digraphs.}  

\changed{Apart from the obvious relevance to the theory of context-free languages, this contribution has the practical potential to enable (4) \emph{generic context-free parsers} that produce different families of non-crossing graphs with the same set of inference rules while the search space in each case is restricted with lexical features and the grammar.}

\paragraph{Outline}

After some background on graphs and parsing as inference (Section 2), we use an \textbf{ontology of digraphs} to illustrate natural families of noncrossing digraphs in Section 3. We then develop, in Section 4, the first \textbf{latent context-free representation} for the set of noncrossing digraphs, then extended in Section 5 with additional latent states supporting our \textbf{finite-state axiomatization} of digraph properties, and allowing us to control the search space using the lexicon.  
The \textbf{experiments} in Section 6 cross-validate our axioms and sample the growth of the constrained search spaces. Section 7 outlines the applications for practical parsing, and Section 8 concludes.

\section{Background}

\paragraph{Graphs and Digraphs}

A \emph{graph} is a pair $(V,E)$ where $V$ is a finite set of vertices and $E \subseteq \{\{u,v\}\subseteq V\}$ is a set of edges.  A sequence of edges of the form $\{v_0,v_1\},$ $\{v_1,v_2\},$ $... ,$ $\{v_{m-1},v_m\}$, with no repetitions in $v_1,...,v_m$, is a \emph{path} between vertices $v_0$ and $v_m$ and \emph{empty} if $m=0$.  A graph is a \emph{forest} if no vertex has a non-empty path to itself and \emph{connected} if all pairs of vertices have a path.  A \emph{tree} is a connected forest.

A \emph{digraph} is a pair $(V,A)$ where
$A \subseteq V \times V$ is a set of arcs $u {\;\rightarrow\;} v$, thus a \emph{directed} graph.  
Its \emph{underlying graph}, $(V,E_A)$, has edges $E_A=\{\{u,v\} \mid (u,v) \in A\}$.  A sequence of arcs $v_0 \rightarrow v_1, v_1 \rightarrow v_2, ... , v_{m-1} \rightarrow v_m$, with no repetitions in $v_1,\dots,v_m$, is a \emph{directed path}, and \emph{empty} if $m=0$.  

A digraph without self-loops $v \rightarrow v$ is \emph{loop-free} (property \textbf{DIGRAPH$_\textrm{LF}$}). We will focus on loop-free digraphs unless otherwise specified, and denote them just by \textbf{DIGRAPH} for brevity. A digraph is d-acyclic (\textbf{ACYC}$_\textrm{D}$), aka a \emph{dag} if no vertex has a non-empty directed path to itself, \emph{u-acyclic} (\textbf{ACYC}$_\textrm{U}$) aka a \emph{m(ixed)-forest} if its underlying graph is a forest, and \emph{weakly connected} (w.c., \textbf{CONN}$_\textrm{W}$) if its underlying graph is connected.

\paragraph{Dependency Parsing}

The \emph{complete digraph $G_S(V,A)$ of a sentence} $S=x_1...x_n$ 
consists of vertices $V=\{1,...,n\}$ and all possible arcs \changed{$A=V\times V - \{(i,i)\}$}.  The vertex $i\in V$ corresponds to the word $x_i$ and the arc $i\rightarrow j \in A$ corresponds to a possible dependency between the words $x_i$ and $x_j$.  

The task of \emph{dependency parsing} is to find a constrained subgraph $G_S'(V,A')$ of the complete digraph $G_S$ of the sentence.  The standard solution is a rooted directed tree called a \emph{dependency tree} or a dag called a \emph{dependency graph}.

\paragraph{Constrained Inference}

In \emph{arc-factored parsing} \cite{McDonald}, each possible arc $i\rightarrow j$ is equipped with a positive weight $w_{ij}$, usually computed as a weighted sum $w_{ij} = \mathbf{w}\cdot \Phi(S,i\rightarrow j)$ where $\mathbf{w}$ is a weight vector and 
$\Phi(\mathbf{x},i\rightarrow j)$ a feature vector extracted from the sentence $\mathbf{x}$, considering the dependency relation from word $x_i$ to word $x_j$. Parsing then consists in finding an arc subset $A'\subseteq A$ that gives us a constrained subgraph $(V,A') \in \text{Constrained}(V,A)$ of the complete digraph $(V,A)$ with maximum sum of arc weights:
\[
(V,A') = \underset{(V,A')\;\in\; \text{Constrained}(V,A)}{\arg\max}\ \ \sum_{i\rightarrow j\in A'} w_{i,j}.
\]
The complexity of this inference task depends on the constraints imposed on the subgraph. Under no constraints, we simply set $A'=A$. Inference over dags is intractable \cite{Guruswami-etal-2011}.  
Efficient solutions are known for projective trees \cite{Eisner96}, various classes of mildly non-projective trees \cite{GomCL2016}, unrestricted spanning trees \cite{McDonald}, and both unrestricted and weakly connected noncrossing dags \cite{Kuhlmann-Jonsson-2015}. 

\paragraph{Parsimony}

Semantic parsers must be able to produce more than projective trees
because the share of projective trees is pretty low (under $3\%$) in semantic graph banks \cite{Kuhlmann-Jonsson-2015}.  However, if we know that the parses have some restrictions, it is better to use them to restrict the search space as much as possible.  

There are two strategies for reducing the search space.  One is to develop a specialized inference algorithm for a particular natural language or family of dags, such as weakly connected graphs \cite{Kuhlmann-Jonsson-2015}.   The other strategy is to control the local complexity of digraphs through lexical categories \cite{Baldridge:2003:MCC:1067807.1067836} or equivalent mechanisms.  This strategy produces a more sensitive model of the language, but requires a principled insight on how the complexity of digraphs can be characterized.

\section{Constraints on the Search Space}

We will now present a classification of digraphs on the basis of their formal properties.

\paragraph{The Noncrossing Property}

For convenience, graphs and digraphs may be ordered like in a complete digraph of a sentence.  Two edges $\{i,j\}$, $\{k,l\}$ in an ordered graph or arcs $i\rightarrow j,k\rightarrow l$ in an
ordered digraph are said to be \emph{crossing} if $\min\{i,j\} < \min\{k,l\} < \max\{i,j\} < \max\{k,l\}$.  A graph or digraph is \emph{noncrossing} if it has no crossing edges or arcs. Noncrossing (di)graphs ($\textbf{NC-(DI)GRAPH}$) are the largest possible (di)graphs that can be drawn on a circle without crossing arcs.  In the following, we assume that all digraphs and graphs are noncrossing. 

\changed{An arc $x \rightarrow y$ is \emph{(properly) covered by} an arc $z \rightarrow t$ if ($\{x,y\}\ne\{z,t\}$) and $\min\{z,t\} \le \min\{x,y\} \le \max\{x,y\} \le \max\{z,t\}$.}

\paragraph{Ontology}

Fig. \ref{lattice} presents an ontology of such families of loop-free noncrossing digraphs \changed{that can be distinguished by digraphs with 5 vertices}.   

In the digraph ontology, a \emph{multitree} \changed{aka mangrove} is a dag with the property of being \emph{strongly unambiguous} (\textbf{UNAMB}$_\textrm{S}$), which asserts that, \changed{given two distinct vertices, there is at most one repeat-free path between them \citep{Lange-1997}.\footnote{A different definition forbids diamonds as minors.}}  A \emph{polytree} \changed{\citep{Rebane-Pearl-1987}} is a \changed{multitree} whose underlying graph is a tree. The \emph{out} property (\textbf{OUT}) of a digraph $(V,E)$ means that no vertex $i\in V$ has two incoming arcs $\{j,k\}\rightarrow i$ s.t. $j\ne k$. 

\newcommand\info[2]{\begin{tabular}{c}#1\\[-1ex]\tiny #2\\[-1ex]\end{tabular}}
\begin{figure}[ht]\centering
\scalebox{.6}{
\!\!\!\!\!\!\!\begin{tikzpicture}[node distance=1.93cm]

\node(X) {\info{\text{NC-DIGRAPH}}{+5460}};
\node(C) [below left  of=X] {\ \ \ \ \info{\textbf{CONN}$_\textrm{W}$}{+43571}};
\node(U) [left of=C] {\info{\textbf{UNAMB}$_\textrm{S}$}{+80}\ };
\node(D) [right of=C] {\ \ \ \ \ \ \info{\textbf{ORIENTED}}{+140}};
\node(UU) [right of=D] {};
\node(UUU) [right of=UU] {};

\draw[dashed](X) -- (C);
\draw(X) -- (U);
\draw(X) -- (U);
\draw(X) -- (D);
\draw(X) -- (D);

\node(F) [below left  of=U] {\info{\textbf{ACYC}$_\text{U}$}{+1200}};
\node(O) [left of=F] {\info{\textbf{OUT}}{+10}\ \ \ };
\node(UC) [right of=F] {\info{w.c.unamb.}{+600}};
\node(CD1) [right of=UC] {};
\node(CD) [right of=CD1] {\info{w.c.or.}{+1160}};
\node(DU) [left of=O] {\ \ \info{unamb.or.}{+80}};
\node(DUX) [below of=UUU] {};
\node(OO) [right of=CD] {};
\node(OOO) [right of=OO] {};
\node(O4) [below right of=OO] {};
\node(A) [right of=CD] {\info{\textbf{ACYC}$_\textrm{D}$}{+840}\ \ \ };

\draw(U) -- (O);
\draw(U) -- (O);
\draw(U) -- (F);
\draw(U) -- (F);
\draw[dashed](U) -- (UC);
\draw[dashed](C) -- (UC);
\draw[dashed](C) -- (CD);
\draw[dashed](D) -- (CD);
\draw(D) -- (A);
\draw(D) -- (A);

\node(OD) [below left of=O] {\info{out oriented\ \ \ \ \ \ \ }{+130}};
\node(DD) [below right of=X] {};
\draw[dotted](DD) -- (DU);
\draw(U) -- (DU);
\draw(U) -- (DU);
\draw(DU) -- (OD);
\draw(DU) -- (OD);
\draw(O) -- (OD);

\node(OF) [below left of=F] {\info{out m-forest}{+435}\ \ };

\node(FC) [right of=OF] {\info{\ \ \ \ \ mixed tree}{+3355}};
\node(AU) [below left of=A] {\ \ \ \info{multitree}{+10}};
\node(CA) [right of=AU] {\ \ \info{w.c.dag}{+2960}};
\node(UCD5) [below right of=CD] {};
\node(UCD4) [left of=UCD5] {\info{\ \  \ \ \ \ \ \ \ \ }{}};
\node(UCD) [below left of=CD] {\info{\ \ \ w.c.unamb.or.\ \ }{+370}\ \ \ };

\draw(O) -- (OF);
\draw(O) -- (OF);
\draw(F) -- (OF);
\draw(F) -- (OF);
\draw[dashed](F) -- (FC);
\draw[dashed](UC) -- (FC);
\draw[dashed](UC) -- (UCD);
\draw[dashed](CD) -- (UCD);
\draw[dashed](CD) -- (CA);
\draw(A) -- (CA);
\draw(A) -- (CA);
\draw(A) -- (AU);
\draw(A) -- (AU);
\draw(U) -- (AU);
\draw(U) -- (AU);

\node(OFC) [below left of=FC] {\info{out mixed tree}{+220}\ \ \ };
\draw[dashed](OF) -- (OFC);
\draw[dashed](OF) -- (OFC);
\draw[dashed](FC) -- (OFC);

\node(UCD2) [below left of=UCD] {};
\node(UCD3) [below left of=UCD2] {};
\node(OCD) [below left of=OFC] {\info{w.c. out oriented}{+132}};
\draw(OD) -- (OCD);
\draw(OD) -- (OCD);
\draw[dotted](UCD4) -- (OCD);

\node(CAU) [below left of=CA] {\ \ \ \ \ \ \info{w.c.multitree}{+50}};
\draw(CA) -- (CAU);
\draw(CA) -- (CAU);
\draw[dashed](UCD) -- (CAU);
\draw(AU) -- (CAU);
\draw(AU) -- (CAU);
\node(FAU) [left of=CAU] {\info{or.forest\ \ \ \ \ }{+300}};
\draw(F) -- (FAU);
\draw(F) -- (FAU);
\draw(AU) -- (FAU);
\draw(AU) -- (FAU);

\node(FCAU) [below left of=CAU] {\ \ \ \ \ \ \info{polytree}{+605}};
\draw(CAU) -- (FCAU);
\draw(CAU) -- (FCAU);
\draw(FAU) -- (FCAU);
\draw(FAU) -- (FCAU);
\draw[dashed](FC) -- (FCAU);

\node(OFAU) [left of=FCAU] {\info{out or.forest\vphantom{j}}{+481}\ \ \ };
\draw(OF) -- (OFAU);
\draw(OF) -- (OFAU);
\draw(FAU) -- (OFAU);
\draw(FAU) -- (OFAU);

\node(OFCAU) [below left of=FCAU] {\info{out or.tree}{+275}};
\draw(OFAU) -- (OFCAU);
\draw(OFAU) -- (OFCAU);
\draw(FCAU) -- (OFCAU);
\draw(FCAU) -- (OFCAU);
\draw(OFC) -- (OFCAU);
\draw(OFC) -- (OFCAU);
\draw(OCD) -- (OFCAU);
\draw(OCD) -- (OFCAU);

\end{tikzpicture}}\\[-2ex]
\caption{\label{lattice}\changed{Basic properties split the set of 62464 noncrossing digraphs for 5 vertices into 23 classes}}
\end{figure}
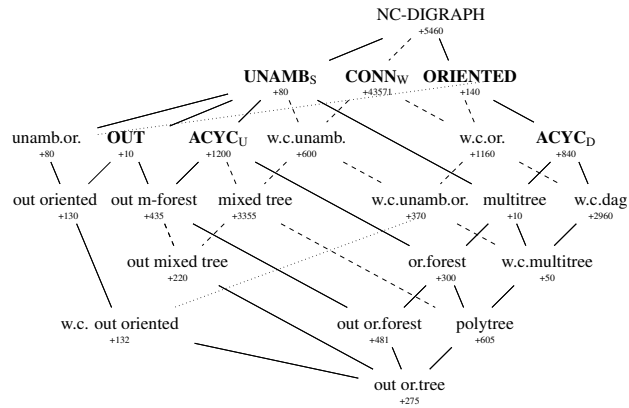

\newcommand\PROJW{\textbf{PROJ}$_\textrm{W}$}

\changed{An ordered digraph is \emph{weakly projective} (\PROJW) if for all vertices $i$, $j$ and $k$, if $k\rightarrow j \rightarrow i$, then either $\{i,j\} < k$ or $\{i,j\} > k$. In other words, the constraint, aka the \emph{outside-to-inside constraint} \citep{DBLP:journals/ijfcs/Yli-Jyra05}, states that no outgoing arc of a vertex properly covers an incoming arc.  This is implied by a stronger constraint known as \emph{Harper, Hays, Lecerf and Ihm projectivity} \citep{Marcus1967}.}

We can embed the ontology of graphs (unrestricted, connected, forests and trees) into the ontology of digraphs by viewing an undirected graph $(V,E)$ as an inverse digraph $(V,\{(i,j),(j,i) \mid \{i,j\} \in E\})$.   This kind of digraph has an \emph{inverse property} (\textbf{INV}).  Its opposite is an \emph{oriented (or.)} digraph $(V,A)$ where $i\rightarrow j\in A$ implies $j\rightarrow i \notin A$ (defines the property \textbf{ORIENTED}).
Out forests and trees are, by convention, oriented digraphs with an underlying forest or tree, respectively.

\paragraph{Distinctive Properties}
A few important properties of digraphs are local and can be verified by inspecting each vertex separately with its incident arcs.  These include (i) the out property (\textbf{OUT}), (ii)  the nonstandard projectivity property (\PROJW), (iii) the inverse property (\textbf{INV}) and (iv) the orientedness (or.) property.  

Properties \textbf{UNAMB}$_\textrm{S}$, \textbf{ACYC}$_\textrm{D}$, \textbf{CONN}$_\textrm{W}$, and \textbf{ACYC}$_\textrm{U}$ are nonlocal properties of digraphs and cannot be generally verified locally, through finite spheres of vertices \cite{Gradel:2005:FMT:1206819}.  
The following proposition covers the configurations that we have to detect in order to decide the nonlocal properties of noncrossing digraphs.

\begin{proposition}\label{prop}
Let $G=(V,E)$ be a noncrossing digraph.  
\begin{itemize}\setlength\itemsep{0ex}
\item If $G \notin \textbf{UNAMB}_\textrm{S}$, then the digraph contains one of the following \changed{four} configurations or their reversals:\\[-2ex]
\hfill $\scalebox{.8}{
    \begin{dependency}[theme = simple, edge style=->]
      \begin{deptext}[column sep=1em]
        \small u \& \small v \& \small y   \\
  \end{deptext}
      \depedge[edge style=dashed,arc angle=50,edge style=->]{1}{2}{}
      \depedge[edge style=dashed,edge start x offset=1pt,edge style=<-,arc angle=50]{3}{2}{}
      \depedge[edge start x offset=-4pt,style=->]{1}{3}{}
\end{dependency}}\!\!\!$%
$\scalebox{.8}{
    \begin{dependency}[theme = simple, edge style=->]
      \begin{deptext}[column sep=1em]
        \small u \& \small v \& \small y   \\
  \end{deptext}
      \depedge[edge style=dashed,arc angle=50,edge style=<-]{1}{2}{}
      \depedge[edge style=dashed,edge start x offset=1pt,edge style=<-,arc angle=50]{3}{2}{}
      \depedge[edge start x offset=-4pt,style=->]{1}{3}{}
\end{dependency}}\!\!\!$%
$\scalebox{.8}{
    \begin{dependency}[theme = simple, edge style=->]
      \begin{deptext}[column sep=1em]
        \small u \& \small v \& \small y  \& \small 
\\
  \end{deptext}
      \depedge[edge style=dashed,arc angle=50,edge style=->]{1}{2}{}
      \depedge[edge style=dashed,edge start x offset=-3pt,edge style=<-,arc angle=50]{2}{3}{}
      \depedge[edge start x offset=-3pt,style=->]{1}{3}{}
\end{dependency}}\!\!\!\!\!\!\!\!\!\!\!\!$%
$\scalebox{.8}{
    \begin{dependency}[theme = simple, edge style=->]
      \begin{deptext}[column sep=1em]
        \small u \& \small v \& \small x \& \small y \& \small 
\\
  \end{deptext}
      \depedge[edge style=dashed,arc angle=50,edge style=<-]{1}{2}{}
      \depedge[edge style=dashed,edge start x offset=1pt,edge style=<-,arc angle=50]{3}{2}{}
      \depedge[edge style=dashed,arc angle=50,edge style=->]{4}{3}{}
      \depedge[edge start x offset=-4pt,style=->]{1}{4}{}
\end{dependency}}$\\[-4ex]%

\item If $G \notin \textbf{ACYC}_\textrm{D}$, then
the graph contains one of the \changed{configurations}\\[-1ex]
\hfill $\scalebox{.8}{
    \begin{dependency}[theme = simple, edge style=->]
      \begin{deptext}[column sep=1em]
        \small u \& \small v \& \small y   \\
  \end{deptext}
      \depedge[edge style=dashed,arc angle=50,edge style=->]{1}{2}{}
      \depedge[edge style=dashed,edge start x offset=1pt,edge style=<-,arc angle=50]{3}{2}{}
      \depedge[edge start x offset=-4pt,style=<-]{1}{3}{}
\end{dependency}}\ \ 
\scalebox{.8}{
    \begin{dependency}[theme = simple, edge style=<-]
      \begin{deptext}[column sep=1em]
        \small u \& \small v \& \small y   \\
  \end{deptext}
      \depedge[edge style=dashed,arc angle=50,edge style=<-]{1}{2}{}
      \depedge[edge style=dashed,edge start x offset=1pt,edge style=->,arc angle=50]{3}{2}{}
      \depedge[edge start x offset=-4pt,style=->]{1}{3}{}
\end{dependency}}\ 
\scalebox{.8}{
    \begin{dependency}[theme = simple, edge style=->]
      \begin{deptext}[column sep=1em]
        \small u \& \small v   \& \small \\
  \end{deptext}
      \depedge[arc angle=90,edge style=<->]{1}{2}{}
\end{dependency}}$\\[-4ex]

\item If $G \notin \textbf{ACYC}_\textrm{U}$, then
the underlying graph contains the following configuration:\\[-1ex]
$\scalebox{.8}{
    \begin{dependency}[theme = simple, edge style=->]
      \begin{deptext}[column sep=1em]
        \small u \& \small v \& \small y \& \small \\
  \end{deptext}
      \depedge[edge style=dashed,arc angle=50,edge style=-]{1}{2}{}
      \depedge[edge style=dashed,edge start x offset=1pt,edge style=-,arc angle=50]{3}{2}{}
      \depedge[edge start x offset=-4pt,style=-]{1}{3}{}
\end{dependency}}$\\[-4ex]
\item If $G \notin \textbf{CONN}_\textrm{W}$, then
the underlying graph contains one of the following configurations:\\
$\scalebox{.8}{
    \begin{dependency}[theme = simple, edge style=->]
      \begin{deptext}[column sep=1em]
        \small ... \& \small v \ \ \& \small y \& \small \ \ ...   \\
  \end{deptext}
      \depedge[edge style=dotted,arc angle=50,edge style=-]{1}{2}{\normalsize   no arc\phantom{xxxx}}
      \depedge[edge style=dashed,edge style=-]{2}{3}{}
      \depedge[edge style=dotted,arc angle=50,edge style=-]{3}{4}{\normalsize \phantom{xxxx}no arc}
\end{dependency}}\ \ \ \ \ 
\scalebox{.8}{
    \begin{dependency}[theme = simple, edge style=->]
      \begin{deptext}[column sep=1em]
        \small ... \& \small v \& \small \ \ ...   \\
  \end{deptext}
      \depedge[edge style=dotted,arc angle=50,edge style=-]{1}{2}{\normalsize no arc\phantom{xxxx}}
      \depedge[edge style=dotted,arc angle=50,edge style=-]{2}{3}{\normalsize \phantom{xxxx}no arc}
\end{dependency}}$\\[-3ex]
\end{itemize}
\end{proposition}
Proposition \ref{prop} gives us a means to implement the property tests in practice.  It tells us intuitively that although the paths can be arbitrarily long, any underlying cycle containing more than 2 arcs consists of one covering arc and a linear chain of edges between its end points.

\section{The Set of Digraphs as a Language}

In this section, we show that the set of noncrossing digraphs is isomorphic to an unambiguous context-free language over a bracket alphabet.

\subsection{Basic Encoding}

Any noncrossing ordered graph $([1,...,n],E)$, even with self-loops, can be encoded
as a string of brackets using the algorithm \texttt{enc} in 
Fig. \ref{encdec}.  For example, the output for the ordered graph
\\
{\centering\scalebox{.7}{{
    \begin{dependency}[theme = simple, edge style=-]
      \begin{deptext}[column sep=1em]
        1 \& \selfloop{2} \& 3  \& 4
        \\
  \end{deptext}
      \depedge{1}{2}{}
      \depedge{2}{4}{}
      \depedge{1}{4}{}
  \end{dependency}}
} \ \ \ \ \ \raisebox{2.5ex}{\small\ensuremath{\begin{matrix}
n = 4, \ \ \ 
E = \left\{\begin{matrix}\{1,2\},&\{2,2\}\\
                         \{2,4\},&\{1,4\}\end{matrix}
    \right\}\end{matrix}}}
}\\ 
is the string \texttt{\small{} [[\{\}][[]\{\}\{\}]]}. Intuitively, pairs of brackets of the form \texttt{\small{\{\}}} can be interpreted as spaces between vertices, and then each set of matching brackets \texttt{\small{[...]}} encodes an arc that covers the  spaces represented inside the brackets.

\begin{figure}[t]
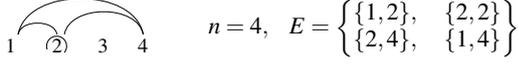

\centering
{\scriptsize\begin{verbatim}
func enc(n,E):                func dec(stdin):
  for i in [1,...,n]:           n = 1; E = {}; s = []
     for j in [i-1,...,2,1]:    while c in stdin:
        if {j,i} in E:             if c == "[":
           print "]"                  s.push(n)
     for j in [n,n-1,...,i+1]:     if c == "]":   
        if {i,j} in E:                i = s.pop()
           print "["                  E.insert((i,n)) 
     if {i,i} in E:                if c == "{":   
        print "[]"                    n = n + 1
     if i<n:                    return (n,E)
        print "{}"                          
\end{verbatim}}\vspace{-2ex}
\caption{\label{encdec}\changed{The encoding and decoding algorithms}}
\end{figure}
Any noncrossing ordered digraph $([1,\dots,n],A)$ 
can be encoded with slight modifications to the algorithm.  
Instead of printing \texttt{\small [ ]} for an edge $\{i,j\}\in E_A$, $i\le j$,
the algorithm should now print 
\[
\textrm{\small\begin{tabular}{cl}
    \texttt{\small / >} & if $(i,j)\in A$,$(j,i)\not\in A$; \\
    \texttt{\small < \backslash} & if $(i,j)\notin A$,$(j,i)\in A$; \\
    \texttt{\small [ ]} & if $(i,j),(j,i) \in A$.
\end{tabular}}
\]
In this way, we can simply encode the digraph
$(\{1,2,3,4\},\{(1,2),(2,2),(4,1),(4,2)\})$ as the string
\texttt{\small </\{\}><[]\{\}\{\}\backslash\backslash}.

\begin{proposition}
  The encoding respects concatenation where the
adjacent nonempty operands have a common vertex.
\end{proposition}

\paragraph{Context-Freeness}

Arbitrary strings with balanced brackets form a context-free language that is known, \changed{generically}, as a Dyck language.  It is easy to see that the graphs $\textbf{NC-GRAPH}$ are encoded with strings that belong to the Dyck language $D_2$ generated by the context-free grammar: \changed{$S \rightarrow \texttt{[} S \texttt{]} S \mid \texttt{\{$S$\}} S\mid \epsilon$.}  The \emph{encoded graphs},  $L_{\textbf{NC-GRAPH}}$, are, however, generated exactly by the context-free grammar $S \rightarrow \texttt{[} S' \texttt{]}\; S \mid \texttt{\{\}}\; S \mid \epsilon$,\;  $S' \rightarrow \texttt{[} S' \texttt{]} \; T \mid \texttt{\{\}} \; S,\; T \rightarrow \texttt{[} S' \texttt{]}\; S\mid \texttt{\{\}} \; S$.  
\changed{This language is an \emph{unambiguous} context-free language.}

\begin{proposition}
The encoded graphs, $L_{\textbf{NC-GRAPH}}$, make an unambiguous context-free language.  \label{uacfl}
\end{proposition}
The practical significance of Proposition \ref{uacfl} is that there is a bijection between $L_{\textbf{NC-GRAPH}}$ and the derivation trees of a context-free grammar.
 
\subsection{Bracketing Beyond the Encoding}

\paragraph{Non-Derivational Representation}

A non-derivational representation for any context-free language $L$ has been given by Chomsky and Sch\"utzenberger \shortcite{CS1963}.  This replaces the stack with a Dyck language $D$ and the grammar rules with co-occurrence patterns specified by a regular language $Reg$.  To hide the internal alphabet from the strings of the represented language, there is a homomorphism that cleans the internal strings of $\emph{Reg}$ and $D$ from internal  markup to get actual strings of the target language:
\[
L_{\textbf{NC-GRAPH}} = h( D \cap \emph{Reg} ).
\]
To make this concrete, replace the previous context free grammar by \changed{$S'' \rightarrow \texttt{[$'$} S' \texttt{]$'$}\; S \mid \texttt{\{\}}\; S \mid \epsilon$, \ $S \rightarrow \texttt{[} S' \texttt{]} \; S \mid \texttt{\{\}} \; S \mid \epsilon$, \ $S' \rightarrow \texttt{[$'$} S' \texttt{]$'$} \; T \mid \texttt{\{\}} \; S,\ T \rightarrow \texttt{[} S' \texttt{]}\; S \mid \texttt{\{\}}\; S$.}  The homomorphism $h$ (Fig. \ref{Rcomp}a) would now relate this language to the original language, mapping the string
\texttt{\small[$'$[$'$\{\}]$'$[[$'$\{\}]$'$\{\}]]$'$} to the string 
\texttt{\small[[\{\}][[\{\}]\{\}]]}, for example.  The Dyck language $D=D_3$ checks that the internal brackets are balanced, and the regular component $Reg$ 
(Fig. \ref{Rcomp}b) checks that the new brackets are used correctly.
A similar representation for the language $L_\textbf{NC-DIGRAPH}$ of encoded digraphs can be obtained with straightforward extensions.
\begin{figure}[th]\centering
\begin{tabular}{@{}c@{}c@{}}
\raisebox{1ex}{\resizebox{0.26\columnwidth}{!}{\includegraphics{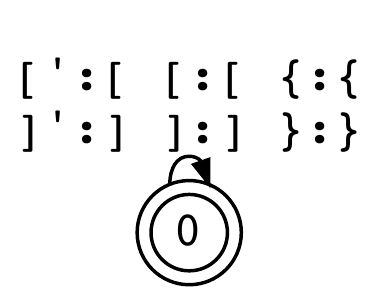}}}\ \  & \ \ \resizebox{0.5\columnwidth}{!}{\includegraphics{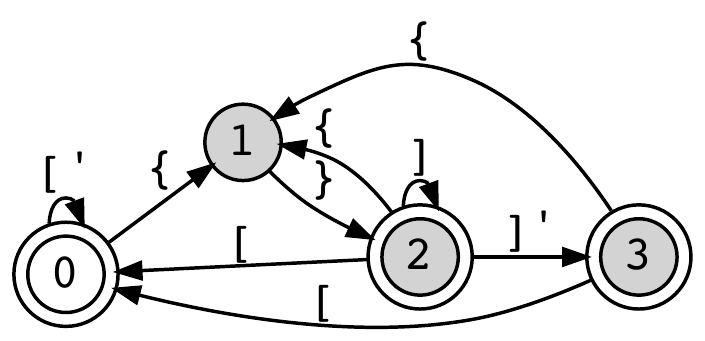}}\\[-1ex]
(a) & (b)\\[-2ex]
\end{tabular}
\caption{\label{Rcomp}\changed{The $h$ and $\emph{Reg}$ components}}
\end{figure}

\changed{The representation $L=h(D\cap Reg)$ is \emph{unambiguous}
if, for every word $w\in L$, the preimage $h^{-1}(w) \cap D\cap Reg$ is a single string.  This implies that $L$ is an \emph{unambiguous} context-free language.}
\begin{proposition}
The set of encoded digraphs, $L_{\textbf{NC-DIGRAPH}}$, has an  unambiguous representation.
\end{proposition}
\begin{proposition}
\changed{Let $L_i=h(D \cap R_i)$, $i\in \{0,1,2\}$ be unambiguous representations with $R_1,R_2 \subseteq R_0$.  Then $L_3=h(D\cap (R_1 \cap R_2))$ is an unambiguous context-free language and the same as $L_1 \cap L_2$. }
\end{proposition}
\begin{proof}
It is immediate that $L_3 \subseteq L_1 \cap L_2$ and $L_3$ is an unambiguous context-free language.  To show that $L_1 \cap L_2 \subseteq L_3$, take an arbitrary $s \in L_1 \cap L_2$.  Since $R_1,R_2 \subseteq R_0$ there is a unique $s'\in h^{-1}(s)$ such that $s' \in D \cap (R_1 \cap R_2)$.  Thus $s \in L_3$.
\end{proof}

\section{Latent Bracketing}

In this section, we extend the internal strings of the non-derivational representation of $L_\textbf{NC-DIGRAPH}$ in such a way that the configurations given in Proposition \ref{prop} can be detected locally from these.

\paragraph{Classification of Underlying Chains}

\newcommand\UNAMBS{\textbf{UNAMB}$_\textrm{S}$}
\newcommand\ACYCD{\textbf{ACYC}$_\textrm{D}$}
\newcommand\ACYCU{\textbf{ACYC}$_\textrm{U}$}
\newcommand\CONNW{\textbf{CONN}$_\textrm{W}$}

A \emph{maximal linear chain} is a maximally long sequence of one or more \changed{edges} that correspond to an underlying left-to-right path in the underlying graph in such a way that \changed{no edge in this chain is properly covered by an edge that does not properly cover all the edges in the chain.}
For example,
\changed{the graph}\\[-3ex]
\[\scalebox{.82}{\!\!\!{%
    \begin{dependency}[theme = simple, edge style=-]%
      \begin{deptext}[column sep=0em]%
\!\texttt{[$'$[$'$\{}\! \& 
\!\texttt{\}]$'$[[$'$\{}\! \& 
\!\texttt{\}]$'$[\{}\!   \& 
\!\texttt{\}]][[$'$\{}\! \& 
\!\texttt{\}]$'$\{}\! \& 
\!\texttt{\}[\{}\! \& 
\!\texttt{\}]]]$'$[\{}\! \& 
\!\texttt{\}[\{} \& 
\texttt{ \}]\{}\! \& 
\!\texttt{\}]}\! \\
  \end{deptext}
      \depedge[arc angle=50,label style=below]{1}{7}{\small I}
      \depedge[label style=below,arc angle=75]{1}{2}{\small \ \ \ II}
      \depedge[label style=below]{2}{3}{\small III}
      \depedge[label style=below]{2}{4}{\small II}
      \depedge[label style=below]{3}{4}{\small III}
      \depedge[label style=below,arc angle=60]{4}{7}{\small II}
      \depedge[label style=below]{7}{10}{\normalsize I}
      \depedge[arc angle=75,label style=below]{4}{5}{\small IV}
      \depedge[arc angle=75,label style=below]{6}{7}{\small V\ \ }
      \depedge[arc angle=120,label style=below]{8}{9}{\small VI\ \ }
\end{dependency}}}\vspace{-2ex}
\]
contains six maximal linear chains, indicated with their Roman numbers on each arc.  We decide non-local properties of noncrossing digraphs by recognizing maximal linear chains as parts of configurations presented in Proposition \ref{prop}.

\changed{Every \emph{loose} chain (like V and VI) starts with a bracket that is adjacent to a \texttt{\}}-bracket.  Such a chain can contribute only a covering edge to an underlying cycle. In contrast, a bracket with an apostrophe marks the beginning of a \emph{non-loose} chain that can either start at the first vertex, or share starting point with a covering chain. When a non-loose chain is covered, it can be touched twice by a covering edge. The prefixes of chains are classified incrementally, from left to right, with a finite automaton (Figure \ref{succ}).  All states of the automaton are final and correspond to distinct classes of the chains.  These classes are encoded to an extended set of brackets.}

\begin{figure}[ht]
\resizebox{1.03\columnwidth}{!}{\includegraphics{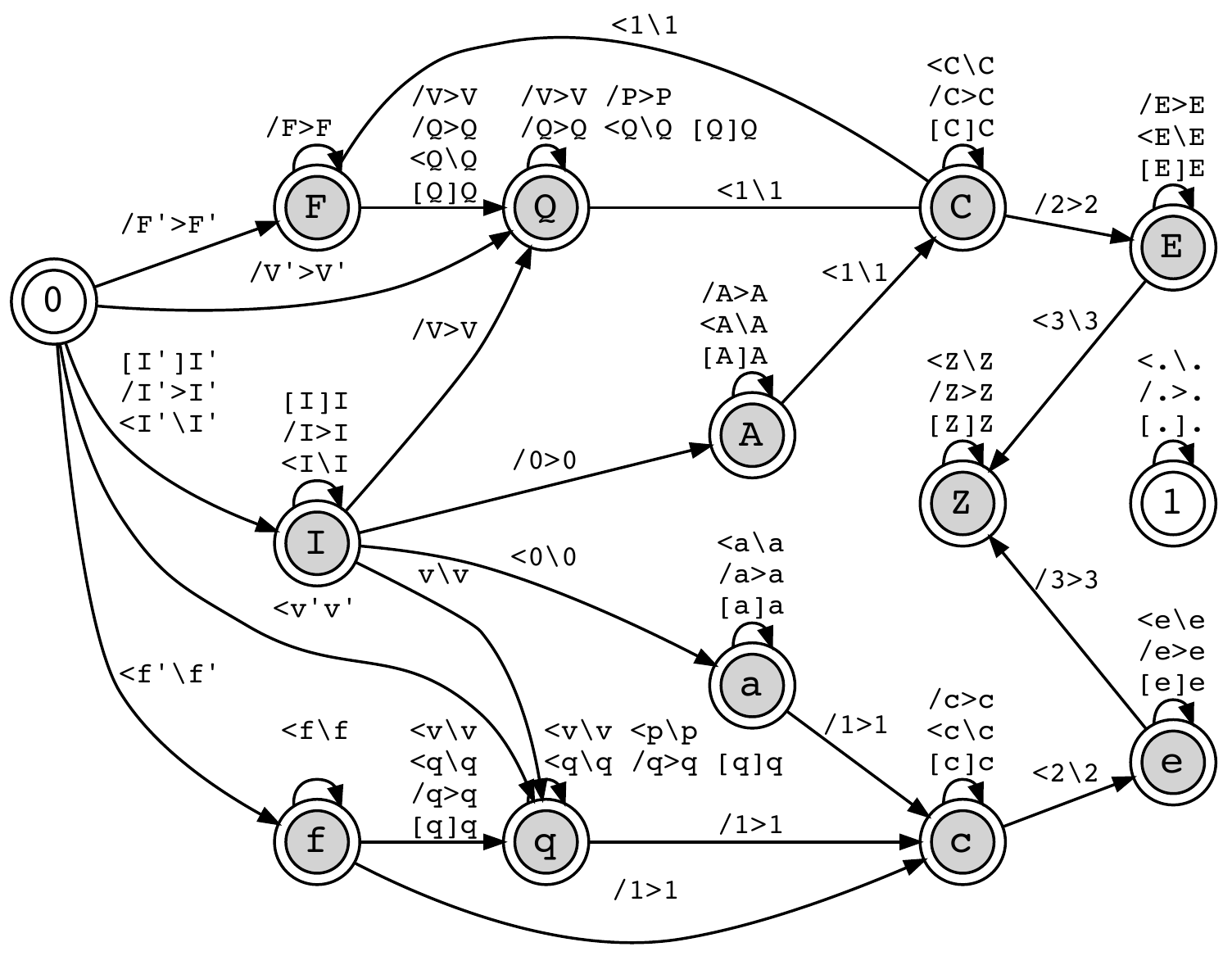}}
\caption{\label{succ}\changed{The finite automaton whose state 0 begins non-loose chains and state 1 loose chains}}
\end{figure}

\changed{The automaton is symmetric: states with uppercase names are symmetrically related with corresponding lowercase states. Thus, it suffices to define the initial and uppercase-named states:}\\

\changed{\small\begin{tabular}{@{\ }ll}
\texttt{0}& the initial state for a non-loose chain; \\
\texttt{I} & a bidirectional chain: \hfill $u\leftrightarrow (v\leftrightarrow) y$;\\
\texttt{A} & a primarily bidirectional forward chain: \hfill $u\leftrightarrow v \rightarrow y$;\\
\texttt{F} & a forward chain: \hfill $u\rightarrow v\rightarrow y$; \\
\texttt{Q} & a primarily forward chain: \hfill $u\rightarrow v \leftrightarrow (\cdots \rightarrow) y$;\\
\texttt{C} & a primarily forward 1-turn chain: \hfill $u\rightarrow v\leftarrow y$;\\
\texttt{E} & a primarily forward 2-turn chain: \hfill $u\rightarrow v \leftarrow x\rightarrow y$;\\
\texttt{Z} & a 3-turn chain;\\
\texttt{1}& the initial (and only) state for a loose chain; \\[1ex]
\end{tabular}}
\changed{Recognition of ambiguous paths in configurations $u\overrightarrow{\overleftarrow{\rightarrow \rightarrow}v\leftarrow}y$ and $u\overrightarrow{\overleftarrow{\leftarrow v\rightarrow x\leftarrow}\leftarrow}y$} involves three chain levels.  To support the recognition, subtypes of edges are defined according to the chains they cover. The brackets {\small \texttt{>I'}, \texttt{\textbackslash I'}, \texttt{>I}, \texttt{\textbackslash I}, \texttt{\textbackslash A}, \texttt{>a}, \texttt{\textbackslash Q}, \texttt{>Q}, \texttt{>q},\texttt{\textbackslash{}q}, \texttt{>C}, \texttt{\textbackslash c}, \texttt{\textbackslash E}, \texttt{>e}} indicate edges that constitute a cycle with the chain they cover.  The brackets {\small \texttt{>V'}, \texttt{\textbackslash v'}, \texttt{>V}, \texttt{\textbackslash v}} indicate edges that cover 2-turn chains. Not all states make these distinctions.

\newcommand\nonloose{\ensuremath{\overline{\textrm{loose}}}}

\newcommand\acycuex{\scalebox{.8}{{
    \begin{dependency}[theme = simple, edge style=-]
      \begin{deptext}[column sep=1em]
       \& \phantom{XXX} \& \phantom{XXX}$R_{\nonloose} R$  \\
  \end{deptext}
      \depedge[arc angle=50]{1}{3}{}
      \depedge[arc angle=0]{2}{3}{\normalsize a \changed{nonloose} chain\phantom{XXXXX}}
      \depedge[edge style=-,edge start x offset=8pt,arc angle=0,edge style=dashed]{1}{2}{}
  \end{dependency}}}}

\newcommand\connwex{\scalebox{.9}{
    \begin{dependency}[theme = simple, edge style=-]
      \begin{deptext}[column sep=1em]
       \& $R_\textrm{loose}\textrm{(no connecting edges)}\!\!\!\!\!\!\!\!\!\!\!\!\!\!\!\!\!\!\!\!\!\!\!\!\!\!\!\!\!\!\!\!\!\!\!\!\!\!\!\!\!\!\!\!\!\!$ \& \& \& \& (a vertex without edges) \\
  \end{deptext}
      \depedge{1}{2}{}
  \end{dependency}}}

\newcommand\acycdex{\scalebox{.8}{{\!\!\!\!
    \begin{dependency}[theme = simple, edge style=->]
      \begin{deptext}[column sep=1em]
       \phantom{XX} \& \phantom{XXXXxXX} \& $\!\!\!\!\!\!\!\!\!\!\!\!\!\!\!\!\!\!\!\!\!\!\!\!R_\textrm{right} R_\texttt{\backslash}$\!\!\!\!\!\!\!\!\!\!\!\!\!\!\!\!\!\!\!\!\!\!\!\!  
       \&
       \phantom{XX} \& \phantom{XxXXXXX} \& $\!\!\!\!\!\!\!\!\!\!\!\!\!\!\!\!\!\!\!\!\!\!\!\!R_\textrm{left} R_\texttt{>}$\!\!\!\!\!\!\!\!\!\!\!\!  
       \& \& \& {\ \ \ \ \ } \\
  \end{deptext}
      \depedge[arc angle=50,edge style=<-,edge start x offset=8pt]{1}{3}{}
      \depedge[edge start x offset=-4pt,arc angle=0]{2}{3}{\normalsize \raisebox{1pt}{forward\phantom{yXXXXXX}}}
      \depedge[arc angle=50,edge start x offset=8pt]{4}{6}{}
      \depedge[edge style=<-,edge start x offset=-4pt,arc angle=0]{5}{6}{\normalsize \raisebox{1pt}{backward\phantom{yXXXXXX}}}
      \depedge[edge style=<->]{8}{9}{inverted arc}
      \depedge[edge style=-,edge start x offset=8pt,arc angle=0,edge style=dashed]{1}{2}{}
      \depedge[edge style=-,edge start x offset=8pt,arc angle=0,edge style=dashed]{4}{5}{}
  \end{dependency}}
}}

\newcommand\unambsex{\scalebox{.8}{{
\!\!\!\!\!\!\!\!\!%
\specialcell{
    \begin{dependency}[theme = simple, edge style=->]
      \begin{deptext}[column sep=1ex]
       \phantom{XX} \& \phantom{xXXXXXX} \& $\!\!\!\!\!\!\!\!\!\!\!\!R_\textrm{right} R_\texttt{>}$\!\!\!\!\!\!\!\!\!\!\!\!  
       \&
       \phantom{XX} \& \phantom{xXXXXXX} \& $\!\!\!\!\!\!\!\!\!\!\!\!R_\textrm{left} R_\texttt{\backslash}$\!\!\!\!\!\!\!\!\!\!\!\!  
       \&
       \phantom{XX} \& \phantom{XXXXXXX} \& $\!\!\!\!\!\!\!\!\!\!\!\!\!\!\!\!\!\!\!\!\!\!\!\!R_\textrm{vergent} R$\!\!\!\!\!\!\!\!\!\!\!\!  \\
  \end{deptext}
      \depedge[arc angle=50,edge start x offset=8pt]{1}{3}{}
      \depedge[edge start x offset=-4pt,arc angle=0]{2}{3}{\normalsize \raisebox{1pt}{forward}\phantom{XXXXX}}
      \depedge[edge style=<-,arc angle=50,edge start x offset=8pt]{4}{6}{}
      \depedge[edge style=<-,edge start x offset=-4pt,arc angle=0]{5}{6}{\normalsize \raisebox{1pt}{backward}\phantom{XXXXX}}
      \depedge[arc angle=50,edge style=-,edge start x offset=8pt]{7}{9}{}
      \depedge[edge style=-,edge start x offset=-4pt,arc angle=00]{8}{9}{\normalsize \raisebox{1pt}{con/divergent}\phantom{XXXXXXX}}
      \depedge[edge style=->,edge start x offset=8pt,arc angle=0,edge style=dashed]{1}{2}{}
      \depedge[edge style=-,edge start x offset=8pt,arc angle=0,edge style=dashed]{7}{8}{}
      \depedge[edge style=<-,edge start x offset=8pt,arc angle=0,edge style=dashed]{4}{5}{}
  \end{dependency}\\[-3ex]
%
\begin{dependency}[theme = simple, edge style=->]
      \begin{deptext}[column sep=1ex]
       \phantom{xxxx} \& \phantom{xx} \& \phantom{xxxxxxx} \& \phantom{xxx}$R_\textrm{left2}R_\texttt{>}$\phantom{}\!\!\!\!\!\!\!\!\!\!\!\! \& 
       \!\!\!\phantom{xxxxxxxxx} \& \phantom{xxx} \& \phantom{x} \& \phantom{xxx}$R_\textrm{right2}R_\texttt{\textbackslash}$\!\!\!\!\!\!\!\!\!\! \\
  \end{deptext}
      \depedge[arc angle=40,edge style=<-,edge start x offset=8pt]{4}{1}{}
      \depedge[edge style=<->, edge start x offset=-4pt,arc angle=0,edge style=dashed]{1}{3}{\normalsize \raisebox{1pt}{\phantom{x}divergent }}
      \depedge[edge style=<-,edge start x offset=4pt,arc angle=0,edge style=dashed]{3}{4}{\normalsize \raisebox{1pt}{backward}\phantom{xxxx}}
      
      \depedge[arc angle=40,edge start x offset=8pt]{8}{5}{}
      \depedge[edge style=->,edge start x offset=-4pt,arc angle=0,edge style=dashed]{5}{6}{\normalsize \raisebox{1pt}{\phantom{xxxx}forward}}
      \depedge[edge style=<->, edge start x offset=4pt,arc angle=0,edge style=dashed]{6}{8}{\normalsize \raisebox{1pt}{divergent\phantom{xx}}}
  \end{dependency}
  }
  }
}}

\newcommand\projex{\scalebox{.9}{{
    \begin{dependency}[theme = simple, edge style=->]
      \begin{deptext}[column sep=1em]
      $L_\texttt{/} L_\texttt{<}$ \& \& \& \& \& \& $R_\texttt{>} R_\texttt{\backslash}$  \\
  \end{deptext}
      \depedge[edge style=->,arc angle=50]{1}{3}{}
      \depedge[edge style=<-,arc angle=50]{1}{2}{}
      \depedge[arc angle=50,edge style=<-]{5}{7}{}
      \depedge[arc angle=50]{6}{7}{}
  \end{dependency}}
}}

\begin{table*}[ht]\centering
\scalebox{.89}{\begin{tabular}{|@{}c@{}cl@{}|}
\hline &&\\[-2ex]
Forbidden patterns in noncrossing digraphs\c & 
Property & 
Constraint language \\
\hline
  \raisebox{-3ex}{\acycuex} &
\ACYCU  &
$A_U= \Sigma^* - \Sigma^* R_{\nonloose}R \Sigma^*$ 
\\[-3.5ex]
\ \ \raisebox{-2ex}{\connwex} &
\CONNW  & 
\changed{$C_W=\Sigma^* 
- \Sigma^* R_\textrm{loose} (\epsilon \cup B \Sigma^*)
- (B \Sigma^* \cup \Sigma^* B)
$}
\\[-1ex]
\raisebox{-3ex}{\acycdex} &
\ACYCD  &
$A_D=\Sigma^* - \Sigma^* ( R_\textrm{right}R_\texttt{\backslash} \cup R_\textrm{left}R_\texttt{>} \cup \Sigma_\textrm{inv} ) \Sigma^*$
\\[-1ex] 
\raisebox{-3ex}{\unambsex}\phantom{xxx}\!\!\!\!\!\!\!\!\!\!\!  &
\UNAMBS & 
$U_S = \begin{matrix}\\[-1ex]
\Sigma^* - \Sigma^* ( R_\textrm{right}R_\texttt{>} \cup R_\textrm{left}R_\texttt{\backslash} \cup R_\textrm{vergent}R ) \Sigma^* 
\\ - \Sigma^*( R_\textrm{left2}R_> \cup R_\textrm{right2}R_\texttt{\textbackslash} ) \Sigma^* 
\end{matrix}$
\\[-2ex] 
\raisebox{-3ex}{\projex} & 
\PROJW & 
$P_W=\Sigma^* - \Sigma^* 
( L_\texttt{/}L_\texttt{<} \cup 
R_\texttt{>}R_\texttt{\backslash} ) \Sigma^*$ 
\\ 
(an arc without inverse) &
\textbf{INV} &  
$I = \Sigma^* - \Sigma^* \Sigma_\textrm{or} \Sigma^*$
\\ 
(a state with more than 2 incoming arcs) &
\textbf{OUT} & $\textit{Out}=\Sigma^* - \Sigma^* \Sigma_\textrm{in} ( \Sigma - B)^* \Sigma_\textrm{in}  \Sigma^*$
\\ 
(an inverted edge) &
\!\!\!\changed{\textbf{ORIENTED}}\!\!\! & 
$O=\Sigma^* - \Sigma^* \Sigma_\textrm{inv}   \Sigma^*$
\\\hline
\end{tabular}}
\caption{\label{family}Properties of encoded noncrossing digraphs as constraint languages}
\end{table*}

\paragraph{Extended Representation}

\changed{The extended brackets encode the latent structure of digraphs: the orientation and the subtype of the edge and the class of the chain.   The total alphabet $\Sigma$ of the strings now contains the boundary brackets \texttt{\small\{\}} and 54 pairs of brackets (Figure \ref{succ}) for edges from which we obtain a new Dyck language, $D_{55}$, and an extended homomorphism $h_\textrm{lat}$.}  

\changed{The \textit{Reg} component of the language representation is replaced with \textit{Reg$_\textrm{lat}$}, that is, an intersection of (1) an inverse homomorphic image of \textit{Reg} to strings over the extended alphabet, (2) a local language that constrains adjacent edges according to Figure \ref{succ}, (3) a local language specifying how the chains start, and (4) a local language that distinguishes pure oriented edges from those that cover a cycle or a 2-turn chain.  The new component requires only 24 states as a deterministic automaton.}

\begin{proposition}
$h_\textrm{lat}(D_{55} \cap \textit{Reg}_\textrm{lat})$ is an unambiguous representation for $L_{\textbf{NC-DIGRAPH}}$.
\end{proposition}
The internal language $L_\textbf{NC-DIGRAPH$_\textrm{lat}$} = D_{55} \cap \textit{Reg}_\textrm{lat}$ is called the set of \emph{latent encoded digraphs}.

\let\amp\&
\paragraph{Example}
Here is a digraph with its latent encoding:\\
$\scalebox{1.25}{\tiny{\!\!\!\!
    \begin{dependency}[theme = simple, edge style=->]
      \begin{deptext}[column sep=0ex]
$\underbrace{\texttt{<f$'$\;[I$'$}}_{1}\texttt{\{}$\!\&%
\!$\texttt{\}}\underbrace{\texttt{]I$'$\;/0\;/F$'$}}_{2}\texttt{\{}$\!\&%
\!$\texttt{\}}\underbrace{\texttt{>F$'$}}_{3}\texttt{\{}$\!\&%
\!$\texttt{\}}\underbrace{\texttt{<$.$}}_{4}\texttt{\{}$\!\&%
\!$\texttt{\}}\underbrace{\texttt{/$.$}}_{5}\texttt{\{}$\!\&%
\!$\texttt{\}}\underbrace{\texttt{>$.$}}_{6}\texttt{\{}$\!\&%
\!$\texttt{\}}\underbrace{\texttt{\backslash$.$\;>0\;\backslash{}f$'$}}_{7}$\\
  \end{deptext}
      \depedge[edge style=<->,arc angle=30]{1}{7}{}
      \depedge[arc angle=20]{1}{2}{}
      \depedge[arc angle=20]{2}{3}{}
      \depedge[arc angle=30]{2}{7}{}
      \depedge[arc angle=30,edge style=<-]{4}{7}{}
      \depedge[arc angle=20]{5}{6}{}
  \end{dependency}}
}$\\[1ex]
The brackets in the extended representation contain information that helps us recognize, through local patterns, that this graph has a directed cycle (directed path $1\rightarrow 2\rightarrow 7\rightarrow 1$), is strongly ambiguous (two directed paths $2\rightarrow 1$ and $2\rightarrow 7\rightarrow 1$) and is not weakly connected (vertices 5 and 6 are not connected to the rest of the digraph).

\paragraph{Expressing Properties via Forbidden Patterns}

We now demonstrate that all the mentioned nonlocal properties of graphs have become local in the extended internal representation of the code strings $L_\textbf{NC-DIGRAPH}$ for noncrossing digraphs.  

These distinctive properties of graph families reduce to  forbidden patterns in bracket strings and then compile into regular constraint languages.  These are presented in Table \ref{family}.   To keep the patterns simple, subsets of brackets are defined:
\begin{center}
\begin{tabular}{l}
\scalebox{.85}{%
\begin{tabular}{p{.9cm}lll}
$L_\texttt{/}$      & \texttt{[}-,\texttt{/}-brackets &
$L_\texttt{<}$      & \texttt{[}-,\texttt{<}-brackets\\ 
$R_\texttt{>}$      & \texttt{]}-,\texttt{>}-brackets &
$R_\texttt{\backslash}$& \texttt{]}-,\texttt{\textbackslash}-brackets\\ 
$B$                 & \texttt{\small\{}, \texttt{\small\}} &
$R$                 & $R_\texttt{>}\cup R_\texttt{\textbackslash}$\\
$R_\text{loose}$    & $\texttt{\}},\; \texttt{>.},\; \texttt{\backslash.},\; \texttt{].}\;$ &
$R_{\nonloose}$  & $R - R_\text{loose}$\\
$R_\textrm{right}$  & $R$ reaching \texttt{\small F$,$Q$,$I$,$A}&
$R_\textrm{left}$   & $R$ reaching \texttt{\small f$,$q$,$i$,$a}\\
$R_\textrm{right2}$  & {\small\texttt{>P}, \texttt{>2}, \texttt{>E}, \texttt{\textbackslash{}E}, \texttt{]E}}
& 
$R_\textrm{left2}$   & {\small\texttt{\textbackslash{}p}, \texttt{\textbackslash2}, \texttt{\textbackslash{}e}, \texttt{>e}, \texttt{]e}} \\
$\Sigma_\textrm{in}$  & $L_\texttt{<}\cup R_\texttt{>}$ &
$\overline{B}$ & $\Sigma - B$
\end{tabular}}\\
\scalebox{.85}{\begin{tabular}{p{.9cm}l}
$R_\textrm{vergent}$& non-\texttt{'} $R$ reaching \texttt{I,Q,q,A,a,C,c}\\
$\Sigma_\textrm{or}$ & all brackets for oriented edges \\
$\Sigma_\textrm{inv}$    & all brackets for inverted edges
\end{tabular}}
\end{tabular}
\end{center}

\section{Validation Experiments}

The current experiments were designed (1) to help in developing the components of $Reg_\textrm{lat}$ and the constraint languages of axiomatic properties, (2) to validate the representation, the constraint languages and their unambiguity, (3) to learn about the ontology and (4) to sample the integer sequences associated with the cardinality of each family in the ontology.

\paragraph{Finding the Components} Representations of $Reg_\textrm{lat}$ were built with scripts written using a finite-state toolkit \cite{FOMA} that supports rapid exploration with regular languages and transducers. 

\paragraph{Validation of Languages}  Our scripts presented alternative approaches to compute languages of encoded digraphs with $n$ vertices up to $n=9$. We also implemented a Python script that enumerated elements of families of graphs up to $n=6$.
The solutions were used to cross-validate one another.

The constraint $G_n=\overline{B}^*(\texttt{\{\}}\overline{B}^*)^{n-1}$ ensures $n$-vertices in encoded digraphs. The finite set of encoded acyclic 5-vertex digraphs was computed with a finite-state approach \cite{YJ2012FSMNLP} that takes the input projection of the composition
\[\small
\text{Id}(Reg_{lat} \cap A_D \cap G_5) \circ T_{55} \circ  T_{55} \circ T_{55} \circ T_{55} \circ T_{55} \circ \text{Id}(\epsilon) 
\]
where $\text{Id}$ defines an identity relation and transducer $T_{55}$ eliminates matching adjacent brackets.  This composition differs from the typical use where the purpose is to construct a regular relation \cite{Kaplan:1994:RMP:204915.204917} or its output projection \cite{Roche:1996:TPF:974697.974706,DBLP:journals/coling/Oflazer03}.

For digraphs with a lot of vertices, we had an option to 
employ a dynamic programming scheme \cite{DBLP:conf/birthday/Yli-Jyra12} that uses weighted transducers.

\paragraph{Building the Ontology} To build the ontology in Figure \ref{lattice} we first found out which combinations of digraph properties co-occur to define distinguishable families of digraphs.  After the nodes of the lattice were found, we were able to see the partial order between these.

\paragraph{Integer Sequences}
We sampled, for important families of digraphs, the prefixes of their related integer sequences.  We found out that each family of graphs is pretty much described by its cardinality, see Table \ref{Fams}.  In many cases, the number sequence was already well documented \cite{OEIS}.

\begin{table*}[ht]\small
  \caption{\label{Fams}Characterizations for some noncrossing families of digraphs and graphs}
  \vspace{1ex}
  \newcommand\Vdigraph{\raisebox{-3.5ex}{$\scalebox{.56}{\begin{dependency}[theme = simple, edge style=->]\begin{deptext}[column sep=1em]1 \& 2 \& 3 \& 4 \& 5 \\
  \end{deptext}
  \depedge{1}{3}{}
  \depedge{1}{4}{}
  \depedge{1}{5}{}
  \depedge{3}{4}{}
  \depedge[edge style=<->]{4}{5}{}
  \end{dependency}}$}}

\newcommand\VWc{\raisebox{-3ex}{$\scalebox{.56}{\begin{dependency}[theme = simple, edge style=->]\begin{deptext}[column sep=1em]1 \& 2 \& 3 \& 4 \& 5 \\
  \end{deptext}
  \depedge{1}{2}{}
  \depedge{1}{3}{}
  \depedge{1}{4}{}
  \depedge{1}{5}{}
  \depedge{4}{5}{}
  \depedge[edge style=<->]{4}{5}{}
  \end{dependency}}$}}

\newcommand\VAd{\raisebox{-3.5ex}{$\scalebox{.56}{\begin{dependency}[theme = simple, edge style=->]\begin{deptext}[column sep=1em]1 \& 2 \& 3 \& 4 \& 5 \\
  \end{deptext}
  \depedge[edge style=<-]{1}{3}{}
  \depedge{1}{4}{}
  \depedge{1}{5}{}
  \depedge[edge style=<-]{4}{5}{}
  \end{dependency}}$}}

\newcommand\VWcDag{\raisebox{-3.5ex}{$\scalebox{.56}{\begin{dependency}[theme = simple, edge style=->]\begin{deptext}[column sep=1em]1 \& 2 \& 3 \& 4 \& 5 \\
  \end{deptext}
  \depedge[edge style=<-]{1}{2}{}
  \depedge{1}{3}{}
  \depedge{1}{4}{}
  \depedge{1}{5}{}
  \depedge[edge style=<-]{4}{5}{}
  \end{dependency}}$}}

\newcommand\VPwDag{\raisebox{-3.5ex}{$\scalebox{.56}{\begin{dependency}[theme = simple, edge style=->]\begin{deptext}[column sep=1em]1 \& 2 \& 3 \& 4 \& 5 \\
          \end{deptext}
          \depedge{1}{2}{}
          \depedge{1}{3}{}
          \depedge{1}{4}{}
          \depedge{1}{5}{}
          \depedge[edge style=<-]{4}{5}{}
      \end{dependency}}$}}

\newcommand\VAdUs{\raisebox{-3.5ex}{$\scalebox{.56}{\begin{dependency}[theme = simple, edge style=->]\begin{deptext}[column sep=1em]1 \& 2 \& 3 \& 4 \& 5 \\
  \end{deptext}
  \depedge[edge style=<-]{2}{3}{}
  \depedge{1}{2}{}
  \depedge{4}{5}{}
  \depedge{1}{5}{}
  \end{dependency}}$}}

\newcommand\VOut{\raisebox{-3.5ex}{$\scalebox{.56}{\begin{dependency}[theme = simple, edge style=->]\begin{deptext}[column sep=1em]1 \& 2 \& 3 \& 4 \& 5 \\
  \end{deptext}
  \depedge[edge style=<-]{1}{2}{}
  \depedge[edge style=<-]{2}{3}{}
  \depedge{1}{4}{}
  \depedge{1}{5}{}
  \depedge{1}{3}{}
  \end{dependency}}$}}

\newcommand\VAuOr{\raisebox{-2ex}{$\scalebox{.53}{\begin{dependency}[theme = simple, edge style=->]\begin{deptext}[column sep=1em]1 \& 2 \& 3 \& 4 \& 5 \\
  \end{deptext}
  \depedge{1}{4}{}
  \depedge{1}{5}{}
  \depedge{3}{4}{}
  \end{dependency}}$}}

\newcommand\VAuOut{\raisebox{-2.5ex}{$\scalebox{.53}{\begin{dependency}[theme = simple, edge style=->]\begin{deptext}[column sep=1em]1 \& 2 \& 3 \& 4 \& 5 \\
  \end{deptext}
  \depedge{1}{5}{}
  \depedge{1}{4}{}
  \depedge{1}{3}{}
  \depedge[edge style=<->]{1}{2}{}
  \end{dependency}}$}}

\newcommand\VAdUsCw{\raisebox{-2.5ex}{$\scalebox{.53}{\begin{dependency}[theme = simple, edge style=->]\begin{deptext}[column sep=1em]1 \& 2 \& 3 \& 4 \& 5 \\
  \end{deptext}
  \depedge{1}{2}{}
  \depedge{1}{3}{}
  \depedge{1}{5}{}
  \depedge[edge style=<-]{3}{4}{}
  \depedge{4}{5}{}
  \end{dependency}}$}}

\newcommand\VAuUsAdCwOr{\raisebox{-3.5ex}{$\scalebox{.53}{\begin{dependency}[theme = simple, edge style=->]\begin{deptext}[column sep=1em]1 \& 2 \& 3 \& 4 \& 5 \\
  \end{deptext}
  \depedge{1}{3}{}
  \depedge{1}{4}{}
  \depedge{1}{5}{}
  \depedge{2}{3}{}
  \end{dependency}}$}}

\newcommand\VAdOutPw{\raisebox{-3.5ex}{$\scalebox{.56}{\begin{dependency}[theme = simple, edge style=->]\begin{deptext}[column sep=1em]1 \& 2 \& 3 \& 4 \& 5 \\
  \end{deptext}
  \depedge{1}{3}{}
  \depedge{1}{4}{}
  \depedge{1}{5}{}
  \end{dependency}}$}}

\newcommand\VAdOut{\raisebox{-3.5ex}{$\scalebox{.56}{\begin{dependency}[theme = simple, edge style=->]\begin{deptext}[column sep=1em]1 \& 2 \& 3 \& 4 \& 5 \\
  \end{deptext}
  \depedge{1}{4}{}
  \depedge{1}{5}{}
  \depedge[edge style=<-]{1}{3}{}
  \end{dependency}}$}}

\newcommand\VInv{\raisebox{-3.5ex}{$\scalebox{.56}{\begin{dependency}[theme = simple, edge style=-]\begin{deptext}[column sep=1em]1 \& 2 \& 3 \& 4 \& 5 \\
  \end{deptext}
  \depedge{1}{3}{}
  \depedge{1}{4}{}
  \depedge{1}{5}{}
  \depedge{3}{4}{}
  \depedge{4}{5}{}
  \end{dependency}}$}}

\newcommand\VAdCwOut{\raisebox{-3.5ex}{$\scalebox{.56}{\begin{dependency}[theme = simple, edge style=->]\begin{deptext}[column sep=1em]1 \& 2 \& 3 \& 4 \& 5 \\
  \end{deptext}
  \depedge{1}{3}{}
  \depedge{1}{4}{}
  \depedge{1}{5}{}
  \depedge[edge style=<-]{1}{2}{}
  \end{dependency}}$}}

\newcommand\VAuInv{\raisebox{-3.5ex}{$\scalebox{.56}{\begin{dependency}[theme = simple, edge style=-]\begin{deptext}[column sep=1em]1 \& 2 \& 3 \& 4 \& 5 \\
  \end{deptext}
  \depedge{1}{3}{}
  \depedge{1}{4}{}
  \depedge{1}{5}{}
  \end{dependency}}$}}

\newcommand\VCwInv{\raisebox{-3.5ex}{$\scalebox{.56}{
    \begin{dependency}[theme = simple, edge style=-]\begin{deptext}[column sep=1em]1 \& 2 \& 3 \& 4 \& 5 \\
  \end{deptext}
  \depedge{1}{2}{}
  \depedge{1}{3}{}
  \depedge{1}{4}{}
  \depedge{4}{5}{}
  \depedge{1}{5}{}
  \end{dependency}}$}}

    \newcommand\VAdCwOutPw{\raisebox{-3.5ex}{$\scalebox{.56}{\begin{dependency}[theme = simple, edge style=->]\begin{deptext}[column sep=1em]1 \& 2 \& 3 \& 4 \& 5 \\
  \end{deptext}
  \depedge{1}{2}{}
  \depedge{1}{3}{}
  \depedge{1}{4}{}
  \depedge{1}{5}{}
  \end{dependency}}$}}

\newcommand\VAuCwInv{\raisebox{-3.5ex}{$\scalebox{.56}{\begin{dependency}[theme = simple, edge style=-]
        \begin{deptext}[column sep=1em]1 \& 2 \& 3 \& 4 \& 5 \\
        \end{deptext}
        \depedge{1}{2}{}
        \depedge{1}{3}{}
        \depedge{1}{4}{}
        \depedge{1}{5}{}
        \end{dependency}}$}}

\newcommand\eg[1]{\raisebox{-3.5ex}{$\scalebox{.56}{\begin{dependency}[theme = simple, edge style=->]\begin{deptext}[column sep=1em]1\&2\&3\&4\&5 \\ \end{deptext}#1\end{dependency}}$}}
\newcommand\edge[3]{\depedge[edge style=-]{#1}{#2}{#3}}

\newcolumntype{R}[1]{>{\RaggedLeft\hspace{0pt}}p{#1}}
\centering
\scalebox{.95}{\begin{tabular}{|p{.55in}@{}P{1.6in}P{.8in}@{}|p{.75in}@{}P{1.7in}@{}p{.8in}@{}|}
    \hline \small Name & \small Sequence prefix for $n=2,3,...$ & \small Example & \small Name & \small Sequence prefix for $n=2,3,...$ & \small Example \\
    \hline \hline

    \scriptsize digraph & \scriptsize \underline{(\textbf{KJ}):} 4,64,1792,\textbf{62464},2437120,101859328\!\!\!\!\!\!\!\!\!\!\!\!\!\!\!\!\!\!\!\!\!\!\!\!\!\!\!
    \newline \scriptsize $h_\textit{lat}(D_{55} \cap G_n \cap \textit{Reg}_\textit{lat})$
    & \eg{\depedge[edge start x offset=-2pt]{1}{4}{}\depedge[edge start x offset=-2pt,edge end x offset=2pt]{1}{3}{}\depedge[edge style=<-]{1}{2}{}\depedge{2}{3}{}\depedge[edge style=<->]{3}{4}{}}
    & \scriptsize weakly projective \newline digraph & \scriptsize 4,36,480,\textbf{7744},138880,2661376
    \newline \scriptsize $h_\textit{lat}(D_{55} \cap G_n \cap \textit{Reg}_\textit{lat} \cap P_W)$
    & \eg{\depedge[edge start x offset=-2pt]{1}{4}{}\depedge[edge start x offset=-2pt]{1}{3}{}\depedge[edge start x offset=-2pt]{1}{2}{}\depedge{2}{3}{}\depedge[edge style=<->]{3}{4}{}}
    \\[-1.3ex]

    \scriptsize w.c.digraph   & \scriptsize 3,54,1539,\textbf{53298},2051406,84339468
    \newline \scriptsize $h_\textit{lat}(D_{55} \cap G_n \cap \textit{Reg}_\textit{lat} \cap C_W)$
    & \eg{\depedge[edge start x offset=-2pt]{1}{5}{}\depedge[edge start x offset=-2pt]{1}{4}{}\depedge[edge start x offset=-2pt]{1}{3}{}\depedge[edge style=<-]{1}{2}{}\depedge{2}{3}{}\depedge{3}{4}{}\depedge[edge start x offset=3pt,edge style=<->,arc angle=30]{5}{4}{}}
    & \scriptsize w.p. w.c.digraph   & \scriptsize 3,26,339,\textbf{5278},90686,1658772
    \newline \scriptsize $h_\textit{lat}(D_{55} \cap G_n \cap \textit{Reg}_\textit{lat} \cap P_W \cap C_W)$
    & \eg{\depedge[edge start x offset=-2pt]{1}{5}{}\depedge[edge start x offset=-2pt]{1}{4}{}\depedge[edge start x offset=-2pt]{1}{3}{}\depedge{1}{2}{}\depedge{2}{3}{}\depedge{3}{4}{}\depedge[edge style=<->,edge start x offset=3pt, arc angle=30]{5}{4}{}}
    \\[-1ex] 

    \scriptsize unamb.digr. & \scriptsize 4,39,529,\textbf{8333},142995,2594378
    \newline \scriptsize $h_\textit{lat}(D_{55} \cap G_n \cap \textit{Reg}_\textit{lat} \cap U_S)$
    & \eg{\depedge[edge start x offset=-4pt]{1}{4}{}\depedge[edge start x offset=-2pt,edge style=<-]{1}{3}{}\depedge[edge style=<->]{1}{2}{}\depedge[edge style=<-]{3}{4}{}}
    & \scriptsize w.p. unamb.digr. & \scriptsize 4,29,275,\textbf{3008},35884,453489
    \newline \scriptsize $h_\textit{lat}(D_{55} \cap G_n \cap \textit{Reg}_\textit{lat} \cap P_W \cap U_S)$
    & \eg{\depedge{1}{5}{}\depedge{1}{2}{}\depedge[edge style=<-]{2}{4}{}\depedge[edge start x offset=2pt,edge style=<->]{2}{3}{}\depedge[edge style=<-]{5}{4}{}}
    \\[-1ex]

    \scriptsize m-forest & \scriptsize 4,37,469,\textbf{6871},109369,1837396,32062711\!\!\!\!\!
    \newline \scriptsize $h_\textit{lat}(D_{55} \cap G_n \cap \textit{Reg}_\textit{lat} \cap A_U)$
    & \eg{\depedge[edge start x offset=-4pt]{1}{4}{}\depedge[edge start x offset=-2pt,edge style=<-]{1}{3}{}\depedge[edge style=<->]{1}{2}{}}
    & \scriptsize w.p. m-forest & \scriptsize 4,29,273,\textbf{2939},34273,421336
    \newline \scriptsize $h_\textit{lat}(D_{55} \cap G_n \cap \textit{Reg}_\textit{lat} \cap P_W \cap A_U)$
    & \eg{\depedge{1}{4}{}\depedge{1}{2}{}\depedge[edge style=<->]{2}{3}{}}
    \\[-1ex]
    
    \scriptsize out digraph & \scriptsize  4,27,207,\textbf{1683},14229,123840,1102365
    \newline \scriptsize $h_\textit{lat}(D_{55} \cap G_n \cap \textit{Reg}_\textit{lat} \cap Out )$
    & \eg{\depedge{1}{3}{}\depedge[edge style=<-]{1}{2}{}\depedge[edge style=<-]{2}{3}{}\depedge[edge style=<->]{4}{5}{}}
    & \scriptsize w.p. out digraph & \scriptsize 4,21,129,\textbf{867},6177,45840,350379
    \newline \scriptsize $h_\textit{lat}(D_{55} \cap G_n \cap \textit{Reg}_\textit{lat} \cap P_W \cap Out )$
    & \eg{\depedge{1}{3}{}\depedge{1}{2}{}\depedge[edge style=<->]{4}{5}{}}
    \\

\if x
    \scriptsize out m-forest & \scriptsize 4,25,181,\textbf{1411},11521,97168,839575
    \newline \scriptsize $h_\textit{lat}(D_{55} \cap G_n \cap \textit{Reg}_\textit{lat} \cap A_U \cap Out)$
    & \VAuOut
    & \scriptsize w.projective \newline out m-forest & \scriptsize 4,21,129,\textbf{867},6177,45840
    \newline \scriptsize $h_\textit{lat}(D_{55} \cap G_n \cap \textit{Reg}_\textit{lat} \cap P_W \cap A_U \cap Out)$
    & 
    \\[-.8ex] 
    XXX NONINV NONPROJ 3.1 OUT M-FOREST / / [ { } ] { } > { } > { }
    XXX NONINV PROJ 3.1 OUT M-FOREST / / { } > { } > { } [ { } ]
    XXX INV PROJ 3.1 OUT M-FOREST [ { } ] { } [ { } ] { }

    \scriptsize out m-tree & \scriptsize \underline{\textbf{(A005809):}} 3,15,84,\textbf{495},3003
    \newline \scriptsize $h_\textit{lat}(D_{55} \cap G_n \cap \textit{Reg}_\textit{lat} \cap A_U \cap C_W \cap Out)$
    & 
    & \scriptsize w.projective \newline out m-tree & \scriptsize 3,11,48,\textbf{231},1183,6324
    \newline \scriptsize $h_\textit{lat}(D_{55} \cap G_n \cap \textit{Reg}_\textit{lat} \cap A_U \cap C_W \cap Out)$
    & 
    \\[1ex]

    NONINV NONPROJ 2.2 W.C.UNAMB / / < [ { } ] { } \\ < { } \\ > { } >
    NONINV PROJ 2.2 W.C.UNAMB / / { } > < [ { } ] { } \\ / { } > >

    XXX INV PROJ 3.2 MIXED-TREE [ { } ] [ { } ] [ { } ] [ { } ]
    XXX INV NONPROJ 3.2 MIXED-TREE [ [ [ [ { } ] { } ] { } ] { } ]

    XXX NONINV PROJ 3.2 MIXED-TREE / / / { } > [ { } ] { } > { } >
    XXX NONINV NONPROJ 3.2 MIXED-TREE / / < [ { } ] { } \\ { } > { } >

   \fi
   \hline
   
    %
    \scriptsize or. digraph & \scriptsize 3,27,405,\textbf{7533},156735,3492639,77539113\!\!\!\!\!\!
    \newline \scriptsize $h_\textit{lat}(D_{55} \cap G_n \cap \textit{Reg}_\textit{lat} \cap O)$
    & \eg{\depedge[edge start x offset=-2pt]{1}{4}{}\depedge[edge start x offset=-2pt]{1}{3}{}\depedge[edge style=<-]{1}{2}{}\depedge{3}{2}{}\depedge[edge style=<-]{4}{3}{}}
    & \scriptsize w.p. or.digraph & \scriptsize see w.p.dag
    \newline \scriptsize $h_\textit{lat}(D_{55} \cap G_n \cap \textit{Reg}_\textit{lat} \cap P_W \cap O)$
    & \raisebox{-2ex}{\scriptsize \ \ \ see w.p.dag}
    \\[-1ex]

    \scriptsize dags & \scriptsize \underline{\textbf{(A246756):}} 3,25,335,\textbf{5521},101551
    \newline \scriptsize $h_\textit{lat}(D_{55} \cap G_n \cap \textit{Reg}_\textit{lat} \cap A_D)$
    & \eg{\depedge[edge start x offset=-2pt]{1}{4}{}\depedge[edge start x offset=-2pt]{1}{3}{}\depedge[edge style=<-]{1}{2}{}\depedge[edge style=<-]{3}{2}{}\depedge[edge style=<-]{4}{3}{}}
    & \scriptsize w.p. dag & \scriptsize 3,21,219,\textbf{2757},38523, 574725, 8967675
    \newline \scriptsize $h_\textit{lat}(D_{55} \cap G_n \cap \textit{Reg}_\textit{lat} \cap P_W \cap A_D)$
    & \eg{\depedge{1}{4}{}\depedge{1}{3}{}\depedge{1}{2}{}\depedge{2}{3}{}\depedge{3}{4}{}}
    \\[-1ex] 

    \scriptsize w.c. dag & \scriptsize \underline{(\textbf{KJ}):} 2,18,242,\textbf{3890},69074,1306466 
    \newline \scriptsize $h_\textit{lat}(D_{55} \cap G_n \cap \textit{Reg}_\textit{lat} \cap A_D \cap C_W)$
    & \eg{\depedge[edge start x offset=-2pt]{1}{5}{}\depedge[edge start x offset=-2pt]{1}{4}{}\depedge[edge start x offset=-2pt]{1}{3}{}\depedge[edge style=<-]{1}{2}{}\depedge[edge style=<-]{3}{2}{}\depedge[edge style=<-]{4}{3}{}\depedge[edge style=<-]{5}{4}{}}
    & \scriptsize w.p. w.c. dag & \scriptsize 2,14,142,\textbf{1706},22554,316998,4480592
    \newline \scriptsize $h_\textit{lat}(D_{55} \cap G_n \cap \textit{Reg}_\textit{lat} \cap P_W \cap A_D \cap C_W)$
    & \eg{\depedge{1}{5}{}\depedge{1}{4}{}\depedge{1}{3}{}\depedge{1}{2}{}\depedge{2}{3}{}\depedge{3}{4}{}\depedge{4}{5}{}}
    \\[-1ex]  

    \scriptsize multitree & \scriptsize 3,19,167,\textbf{1721},19447,233283,2917843
    \newline \scriptsize $h_\textit{lat}(D_{55} \cap G_n \cap \textit{Reg}_\textit{lat} \cap A_D \cap U_S)$
    & \!\!\!\!\!\raisebox{-1.3ex}{\scriptsize \begin{tabular}{l}see oriented forest\\ or w.c. multitree{}\end{tabular}}
    & \scriptsize w.p. multitree & \scriptsize 3,17,129,\textbf{1139},11005,112797,1203595
    \newline \scriptsize $h_\textit{lat}(D_{55} \cap G_n \cap \textit{Reg}_\textit{lat} \cap P_W \cap A_D \cap U_S)$
    & \eg{\depedge{1}{4}{}\depedge{1}{2}{}\depedge[edge start x offset=1pt,edge style=<-]{2}{3}{}\depedge{3}{4}{}}
    \\[-1ex] 

    \scriptsize or.forest & \scriptsize 3,19,165,\textbf{1661},18191,210407,2528777 
    \newline \scriptsize $h_\textit{lat}(D_{55} \cap G_n \cap \textit{Reg}_\textit{lat} \cap A_D \cup A_U)$
    & \eg{\depedge[edge start x offset=-2pt]{1}{4}{}\depedge[edge style=<-]{1}{3}{}\depedge[edge style=<-]{1}{2}{}}
    & \scriptsize w.p. or.forest & \scriptsize 3,17,127,\textbf{1089},10127,99329,1010189
    \newline \scriptsize $h_\textit{lat}(D_{55} \cap G_n \cap \textit{Reg}_\textit{lat} \cap P_W \cap A_D \cup A_U)$
    & \eg{\depedge{1}{4}{}\depedge{1}{2}{}\depedge[edge style=<-]{2}{3}{}}
    \\[-1ex] 

    \scriptsize w.c. multitree & \scriptsize 2,12,98,\textbf{930},9638,105798,1201062
    \newline \scriptsize $h_\textit{lat}(D_{55} \cap G_n \cap \textit{Reg}_\textit{lat} \cap A_D \cap U_S \cap C_W)$
    & \eg{\depedge[edge start x offset=-2pt]{1}{5}{}\depedge[edge start x offset=-2pt]{1}{3}{}\depedge[edge style=<-]{1}{2}{}\depedge[edge style=<-]{3}{4}{}\depedge{4}{5}{}}
    & \scriptsize w.p. w.c. multitree & \scriptsize 2,10,68,\textbf{538},4650,42572,404354
    \newline \scriptsize $h_\textit{lat}(D_{55} \cap G_n \cap \textit{Reg}_\textit{lat} \cap P_W \cap A_D \cap U_S \cap C_W)$
    & \eg{\depedge{1}{5}{}\depedge{1}{4}{}\depedge{1}{2}{}\depedge[edge start x offset=1pt,edge style=<-]{2}{3}{}\depedge{3}{4}{}}
    \\[-1ex]

    \scriptsize out or.forest & \scriptsize 3,16,105,\textbf{756},5738,45088,363221
    \newline \scriptsize $h_\textit{lat}(D_{55} \cap G_n \cap \textit{Reg}_\textit{lat} \cap A_D \cap Out)$
    & \eg{\depedge[edge start x offset=-2pt]{1}{4}{}\depedge[edge start x offset=-2pt]{1}{3}{}\depedge[edge style=<-]{1}{2}{}}
    & \scriptsize w.p. out or.forest & \scriptsize \underline{\textbf{(A003169):}} 3,14,79,\textbf{494},3294,22952 
    \newline \scriptsize $h_\textit{lat}(D_{55} \cap G_n \cap \textit{Reg}_\textit{lat} \cap P_W \cap A_D \cap Out)$
    & \eg{\depedge{1}{2}{}\depedge{1}{3}{}\depedge{1}{4}{}}
    \\[-1ex] 
    
    \scriptsize polytree & \scriptsize \underline{\textbf{(A153231):}}  2,12,96,\textbf{880},8736,91392 
    \newline \scriptsize $h_\textit{lat}(D_{55} \cap G_n \cap \textit{Reg}_\textit{lat} \cap A_D \cap C_W \cap A_U)$
    & \eg{\depedge[edge start x offset=-2pt]{1}{5}{}\depedge[edge start x offset=-2pt]{1}{4}{}\depedge[edge style=<-]{1}{2}{}\depedge[edge style=<-]{1}{3}{}}
    & \scriptsize w.p. polytree & \scriptsize \underline{\textbf{(A027307):}}2,10,66,\textbf{498},4066,34970
    \newline \scriptsize $h_\textit{lat}(D_{55} \cap G_n \cap \textit{Reg}_\textit{lat} \cap P_W \cap A_D \cap C_W \cap A_U)$
    & \eg{\depedge{1}{2}{}\depedge{1}{4}{}\depedge{1}{5}{}\depedge[edge style=<-]{2}{3}{}}
    \\[-1ex] 

    \scriptsize out or.tree & \scriptsize \underline{\textbf{(A174687):}} 2,9,48,\textbf{275},1638,9996 
    \newline \scriptsize $h_\textit{lat}(D_{55} \cap G_n \cap \textit{Reg}_\textit{lat} \cap A_D \cap C_W \cap Out)$
    & \eg{\depedge[edge start x offset=-2pt]{1}{3}{}\depedge[edge start x offset=-2pt]{1}{4}{}\depedge[edge start x offset=-2pt]{1}{5}{}\depedge[edge style=<-]{1}{2}{}}
    & \scriptsize projective\newline out or.tree & \scriptsize \underline{\textbf{(A006013):}} 2,7,30,\textbf{143},728,3876,21318 
    \newline \scriptsize $h_\textit{lat}(D_{55} \cap G_n \cap \textit{Reg}_\textit{lat}  \cap P_W \cap A_D \cap C_W \cap Out)$
    & \eg{\depedge{1}{2}{}\depedge{1}{3}{}\depedge{1}{4}{}\depedge{1}{5}{}}
    \\
    \hline

    \scriptsize graph & \scriptsize \underline{\textbf{(A054726):}} 2,8,48,\textbf{352},2880,25216 
    \newline \scriptsize $h_\textit{lat}(D_{55} \cap G_n \cap \textit{Reg}_\textit{lat} \cap I)$
    &
    \multicolumn{1}{c}{\eg{\edge{1}{2}{}\edge{2}{3}{}\edge{1}{3}{}\edge{3}{4}{}\edge{1}{4}{}}}
    &
    \scriptsize connected graph & \scriptsize \underline{\textbf{(A007297):}} 1,4,23,\textbf{156},1162,9192 
    \newline \scriptsize $h_\textit{lat}(D_{55} \cap G_n \cap \textit{Reg}_\textit{lat} \cap I \cap C_W )$
    & \eg{\edge{1}{2}{}\edge{2}{3}{}\edge{1}{3}{}\edge{3}{4}{}\edge{1}{4}{}\edge{4}{5}{}\edge{1}{5}{}}
    \\[-1ex]
    
    \scriptsize forest &  \scriptsize \underline{\textbf{(A054727):}} 2,7,33,\textbf{181},1083,6854 
    \newline \scriptsize $h_\textit{lat}(D_{55} \cap G_n \cap \textit{Reg}_\textit{lat} \cap I \cap A_U )$
    & \multicolumn{1}{c}{\eg{\edge{1}{2}{}\edge{1}{3}{}\edge{1}{4}{}}}
    &
    \scriptsize tree  & \scriptsize \underline{\textbf{(A001764,YJ):}} 1,3,12,\textbf{55},273,1428,7752 
    \newline \scriptsize $h_\textit{lat}(D_{55} \cap G_n \cap \textit{Reg}_\textit{lat} \cap I \cap A_U \cap C_W)$
    & \eg{\edge{1}{2}{}\edge{1}{3}{}\edge{1}{4}{}\edge{1}{5}{}}
    \\

    \ifx
    \scriptsize n/a & \scriptsize \underline{\textbf{(A001006):}} 2,4,9,\textbf{21},51,127,323
    \newline \scriptsize $h_\textit{lat}(D_{55} \cap G_n \cap \textit{Reg}_\textit{lat} \cap I \cap Out)$
    &   \eg{\edge{1}{2}{}\edge{4}{5}{}}
    &
    &
    &
    \\
    \fi

    \hline
  \end{tabular}}

   \textbf{A} = \cite{OEIS}, \textbf{KJ} = \citet{Kuhlmann2015} or \citet{Kuhlmann-Jonsson-2015}, \textbf{YJ} = \citet{DBLP:conf/birthday/Yli-Jyra12} 
\end{table*}

\section{The Formal Basis of Practical Parsing}

While presenting a practical parser implementation is outside of the scope of this paper, which focuses in the theory, we outline in this section the aspects to take into account when applying our representation to build practical natural language parsers.

\paragraph{Positioned Brackets} 

In order to do inference in arc-factored parsing, we incorporate weights to the representation.  For each vertex in $G_n$, the brackets are decorated with the respective position number.  Then, we define an input-specific grammar representation where each pair of brackets in $D$ gets an arc-factored weight given the directions and the vertex numbers associated with the brackets.  

\paragraph{Grammar Intersection}
We associate, to each $G_n$, a quadratic-size context-free grammar that generates all noncrossing digraphs with $n$ vertices.  This grammar is obtained by computing (or even precomputing) the intersection $D_{55} \cap Reg_\textrm{lat} \cap G_n$ in any order, exploiting the closure of context-free languages under intersection with regular languages \cite{barhillel-etal:1961a}.  The introduction of the position numbers and weights in the Dyck language gives us, instead, a weighted grammar and its intersection \cite{LA92}.  This grammar is a compact representation for a finite set of weighted latent encoded digraphs. Additional constraints during the intersection tailors the grammar to different families of digraphs.

\paragraph{Dynamic Programming} 
The heaviest digraph is found with a dynamic programming algorithm that computes, for each nonterminal in the grammar, the weight of the heaviest subtree.  A careful reader may notice some connections to Eisner algorithm \cite{Eisner}, context-free parsing through intersection \citep{Nederhof-2003}, and a dynamic programming scheme that uses contracting transducers and factorized composition \cite{DBLP:conf/birthday/Yli-Jyra12}.  Unfortunately, space does not permit discussing the connections here.

\paragraph{Lexicalized Search Space}
In practical parsing, we want the parser behavior and the dependency structure to be sensitive to the lexical entries or features of each word.  We can replace the generic vertex description $\overline{B}^*$ in $G_n$ with subsets that depend on respective lexical entries.  Graphical constraints can be applied to some vertices but relaxed for others.   This application of current results gives a principled, graphically motivated solution to lexicalized control over the search space.

\section{Conclusion}

We have investigated the search space of parsers that produce noncrossing digraphs. 
Parsers that can be adapted to different needs are less dependent on artificial assumptions on the search space.
Adaptivity gives us freedom to model how the properties of digraphs are actually distributed in linguistic data. As the adaptive data analysis deserves to be treated in its own right, the current work focuses on the separation of the parsing algorithm from the properties of the search space.  

This paper makes four significant contributions.

\paragraph{Contribution 1: Digraph Encoding}

The paper introduces, for noncrossing digraphs, an encoding that uses brackets to indicate edges.

Bracketed trees are widely used in generative syntax, treebanks and structured document formats.  There are established conversions between phrase structure and projective dependency trees, but the currently advocated edge bracketing is expressive and captures more than just projective dependency trees.  This capacity is welcome as syntactic and semantic analysis with dependency graphs is a steadily growing field.  

The edge bracketing creates new avenues for the study of connections between noncrossing graphs and context-free languages, as well as their recognizable properties. By demonstrating that digraphs can be treated as strings, we suggest that practical parsing to these structures could be implemented with existing methods that restrict context-free grammars to a regular yield language.

\paragraph{Contribution 2: Context-Free Properties}

Acyclicity and other important properties of noncrossing digraphs are expressible as unambiguous context-free sets of encoded noncrossing digraphs.  This facilitates the incorporation of property testing to dynamic programming algorithms that implement exact inference.

Descriptive complexity helps us understand to which degree various graphical properties are local and could be incorporated into efficient dynamic programming during exact inference.  
It is well known that acyclicity and connecticity are not definable in first-order logic (FO)  
while they can be defined easily in monadic second order logic (MSO) \cite{Courcelle-1997}.  MSO involves set-valued variables whose use in dynamic programming algorithms and tabular parsing is inefficient.  MSO queries have a brute force transformation to first-order (FO) logic, but this 
does not generally help either as it is well known that MSO can express intractable problems.

The interesting observation of the current work is that some MSO definable properties of digraphs become local in our extended encoding.  This encoding is linear compared to the size of digraphs: each  string over the extended bracket alphabet encodes a fixed assignment of MSO variables.  The properties of noncrossing digraphs now reduce to properties of bracketed trees with linear amount of latent information that is fixed for each digraph.

A deeper explanation for our observation comes from the fact that the treewidth of noncrossing and other outerplanar graphs is bounded to 2.  When the treewidth is bounded, all MSO definable properties, including the intractable ones, become linear time decidable for individual structures \cite{COURCELLE199012}.  They can also be decided in a logarithmic amount of writable space \cite{Elberfeld:2010:LVT:1917827.1918379}, e.g. with element indices instead of sets.
By combining this insight with Proposition \ref{prop}, we obtain a logspace solution for testing acyclicity of a noncrossing graph (Figure \ref{logspace}).  \begin{figure}[t]
\centering
{\scriptsize\begin{verbatim}
  func noncrossing_ACYCU(n,E):
     for {u,y} in E and u < y:    # covering edge
        [v,p] = [u,u]
        while p != -1:            # chain continues
           [v,p] = [p,-1]
           for vv in [v+1,...,y]: # next vertex
              if {v,vv} in E and {v,vv} != {u,y}:
                 if vv == y:
                    return False  # found cycle uvy
                 p = vv           # find longest edge
     return True                  # acyclic
\end{verbatim}}\vspace{-2ex}
\caption{\label{logspace}Testing \ACYCU{} in logarithmic space}
\end{figure}

Although bounded treewidth is a weaker constraint than so-called bounded treedepth that would immediately guarantee first-order definability \cite{Elberfeld:2016:FMS:2996393.2946799}, it can sometimes turn intractable search problems to dynamic programming algorithms \cite{Akutsu2012}.  In our case, Proposition \ref{prop} 
gave rise to unambiguous context-free subsets of $L_\textbf{NC-DIGRAPH}$.  These can be recognized with 
dynamic programming and used in efficient constrained inference 
when we add vertex indices to the brackets and weights to the grammar of the corresponding Dyck language.

\paragraph{Contribution 3: Digraph Ontology}

The context-free properties of encoded digraphs have elegant nonderivative language representations and they generate a semi-lattice under language intersection.  Although context-free languages are not generally closed under intersection, all combinations of the properties in this lattice are context-free and define natural families of digraphs.  
The nonderivative representations for our axiomatic properties share the same Dyck language $D_{55}$ and homomorphism, but differ in terms of forbidden patterns.  As a consequence, also any conjunctive combination of these two properties shares these components and thus define  a context-free language.  The obtained semilattice is an ontology of families of noncrossing digraphs.

Our ontology contains important families of noncrossing digraphs used in syntactic and semantic dependency parsing: out trees, dags, and weakly connected digraphs.   It shows the entailment between the properties and proves the existence of less known families of noncrossing digraphs such as strongly unambiguous digraphs and oriented graphs, multitrees, oriented forests and polytrees.  These are generalizations of out oriented trees. However, these families can still be weakly projective.  Table \ref{Fams} shows integer sequences obtained by enumerating digraphs in each family.  At least twelve of these sequences are previously known, which indicates that the families are natural.

We used a finite-state toolkit to build the components of the nongenerative language representation for latent encoded digraphs and the axioms.\footnote{The finite-state toolkit scripts and a Python-based graph enumerator are available at \newline
\url{https://github.com/amikael/ncdigraphs} .} 

\paragraph{Contribution 4: Generic Parsing}

The fourth contribution of this paper is to show that parsing algorithms can be separated from the formal properties of their search space.

All the presented families of digraphs can be treated by parsers and other algorithms (e.g. enumeration algorithms) in a uniform manner.  The parser's inference rules can stay constant and the choice of the search space is made by altering the regular component of the language representation. 

The ontology of the search space can be combined with a constraint relaxation strategy, for example, when an out tree is a preferred analysis, but a dag is also possible as an analysis when no tree is strong enough.  The flexibility applies also to dynamic programming algorithms that complement with the encoding and allow inference of best dependency graphs in a family simply by intersection with a weighted CFG grammar for a Dyck language that models position-indexed edges of the complete digraph.

Since the families of digraphs are distinguished by forbidden local patterns, the choice of search space can be made purely on lexical grounds, blending well with lexicalized parsing and allowing possibilities such as choosing, per each word, what kind of structures the word can go with.

\paragraph{Future work}

\changed{We are planning to extend} the coverage of the approach by exploring 1-endpoint-crossing and $MH_k$ trees \cite{Pitler2013,GomCL2016}, and related digraphs --- see \cite{YJ2004COLING,GomCarWeiCL2011}. 
\changed{Properties such as \emph{weakly projective}, \emph{out}, and \emph{strongly unambiguous} prompt further study.}  

\changed{An interesting avenue for future work is to explore higher order factorizations for noncrossing digraphs and the related inference. We would also like to have more insight on the transformation of MSO definable properties to the current framework and to logspace algorithms.}

\section*{Acknowledgements}

\changed{AYJ has received funding as Research Fellow from the Academy of Finland (dec. No 270354 - A Usable Finite-State Model for Adequate Syntactic Complexity) and Clare Hall Fellow from the University of Helsinki (dec. RP 137/2013).  CGR has received funding from the European Research Council (ERC) under the European Union's Horizon 2020 research and innovation programme (grant agreement No 714150 - FASTPARSE) and from the TELEPARES-UDC project (FFI2014-51978-C2-2-R) from MINECO. The comments of Juha Kontinen, Mark-Jan Nederhof and the anonymous reviewers helped to improve the paper.}

\bibliographystyle{acl_natbib}
\bibliography{cubic}

\begin{thebibliography}{}
\expandafter\ifx\csname natexlab\endcsname\relax\def\natexlab#1{#1}\fi

\bibitem[{Akutsu and Tamura(2012)}]{Akutsu2012}
Tatsuya Akutsu and Takeyuki Tamura. 2012.
\newblock \href{https://doi.org/10.1007/978-3-642-32589-2\_10}{A
  polynomial-time algorithm for computing the maximum common subgraph of
  outerplanar graphs of bounded degree}.
\newblock In Branislav Rovan, Vladimiro Sassone, and Peter Widmayer, editors,
  {\em Mathematical Foundations of Computer Science 2012: 37th International
  Symposium, MFCS 2012, Bratislava, Slovakia, August 27-31, 2012.
  Proceedings\/}, Springer Berlin Heidelberg, Berlin, Heidelberg, pages 76--87.
\newblock
  \href{https://doi.org/10.1007/978-3-642-32589-2\_10}{https://doi.org/10.1007/978-3-642-32589-2\_10}.

\bibitem[{Baldridge and Kruijff(2003)}]{Baldridge:2003:MCC:1067807.1067836}
Jason Baldridge and Geert-Jan~M. Kruijff. 2003.
\newblock \href{https://doi.org/10.3115/1067807.1067836}{Multi-modal
  combinatory categorial grammar}.
\newblock In {\em Proceedings of EACL'03: the Tenth Conference on European
  Chapter of the Association for Computational Linguistics - Volume 1\/}.
  Association for Computational Linguistics, Budapest, Hungary, pages 211--218.
\newblock
  \href{https://doi.org/10.3115/1067807.1067836}{https://doi.org/10.3115/1067807.1067836}.

\bibitem[{Bar-Hillel et~al.(1961)Bar-Hillel, Perles, and
  Shamir}]{barhillel-etal:1961a}
Yehoshua Bar-Hillel, Micha Perles, and Eliahu Shamir. 1961.
\newblock On formal properties of simple phrase structure grammars.
\newblock {\em Zeitschrift f\"ur Phonologie, Sprachwissenschaft und
  Kommunikationsforschung\/} 14:113--124.

\bibitem[{Chomsky and Sch\"utzenberger(1963)}]{CS1963}
Noam Chomsky and Marcel-Paul Sch\"utzenberger. 1963.
\newblock The algebraic theory of context-free languages.
\newblock {\em Computer Programming and Formal Systems\/} pages 118--161.

\bibitem[{Courcelle(1990)}]{COURCELLE199012}
Bruno Courcelle. 1990.
\newblock \href{https://doi.org/10.1016/0890-5401(90)90043-H}{The monadic
  second-order logic of graphs. {I}. recognizable sets of finite graphs}.
\newblock {\em Information and Computation\/} 85(1):12 -- 75.
\newblock
  \href{https://doi.org/10.1016/0890-5401(90)90043-H}{https://doi.org/10.1016/0890-5401(90)90043-H}.

\bibitem[{Courcelle(1997)}]{Courcelle-1997}
Bruno Courcelle. 1997.
\newblock The expression of graph properties and graph transformations in
  monadic second-order logic.
\newblock In G.~Rozenberg, editor, {\em Handbook of Graph Grammars and
  Computing by Graph Transformations\/}, World Scientific, New-Jersey, London,
  volume~1, chapter~5, pages 313--400.

\bibitem[{Eisner(1996)}]{Eisner96}
Jason Eisner. 1996.
\newblock \href{http://aclweb.org/anthology/C/C96/C96-1058.pdf}{Three new
  probabilistic models for dependency parsing: An exploration}.
\newblock In {\em Proceedings of the 16th International Conference on
  Computational Linguistics (COLING-96)\/}. Copenhagen, Denmark, pages
  340--345.
\newblock
  \href{http://aclweb.org/anthology/C/C96/C96-1058.pdf}{http://aclweb.org/anthology/C/C96/C96-1058.pdf}.

\bibitem[{Eisner and Satta(1999)}]{Eisner}
Jason Eisner and Giorgio Satta. 1999.
\newblock \href{https://doi.org/10.3115/1034678.1034748}{Efficient parsing for
  bilexical context-free grammars and {H}ead {A}utomaton {G}rammars}.
\newblock In {\em Proceedings of the 37th Annual Meeting of the Association for
  Computational Linguistics\/}. Association for Computational Linguistics,
  College Park, Maryland, USA, pages 457--464.
\newblock
  \href{https://doi.org/10.3115/1034678.1034748}{https://doi.org/10.3115/1034678.1034748}.

\bibitem[{Elberfeld et~al.(2016)Elberfeld, Grohe, and
  Tantau}]{Elberfeld:2016:FMS:2996393.2946799}
Michael Elberfeld, Martin Grohe, and Till Tantau. 2016.
\newblock \href{https://doi.org/10.1145/2946799}{Where first-order and monadic
  second-order logic coincide}.
\newblock {\em ACM Trans. Comput. Logic\/} 17(4):25:1--25:18.
\newblock
  \href{https://doi.org/10.1145/2946799}{https://doi.org/10.1145/2946799}.

\bibitem[{Elberfeld et~al.(2010)Elberfeld, Jakoby, and
  Tantau}]{Elberfeld:2010:LVT:1917827.1918379}
Michael Elberfeld, Andreas Jakoby, and Till Tantau. 2010.
\newblock \href{https://doi.org/10.1109/FOCS.2010.21}{Logspace versions of the
  theorems of {B}odlaender and {C}ourcelle}.
\newblock In {\em Proceedings of the 2010 IEEE 51st Annual Symposium on
  Foundations of Computer Science\/}. IEEE Computer Society, Washington, DC,
  USA, FOCS '10, pages 143--152.
\newblock
  \href{https://doi.org/10.1109/FOCS.2010.21}{https://doi.org/10.1109/FOCS.2010.21}.

\bibitem[{G\'{o}mez-Rodr\'{i}guez(2016)}]{GomCL2016}
Carlos G\'{o}mez-Rodr\'{i}guez. 2016.
\newblock \href{https://doi.org/10.1162/COLI\_a\_00267}{Restricted
  non-projectivity: Coverage vs. efficiency}.
\newblock {\em Computational Linguistics\/} 42(4):809--817.
\newblock
  \href{https://doi.org/10.1162/COLI\_a\_00267}{https://doi.org/10.1162/COLI\_a\_00267}.

\bibitem[{G{\'{o}}mez{-}Rodr{\'{\i}}guez
  et~al.(2011)G{\'{o}}mez{-}Rodr{\'{\i}}guez, Carroll, and
  Weir}]{GomCarWeiCL2011}
Carlos G{\'{o}}mez{-}Rodr{\'{\i}}guez, John~A. Carroll, and David~J. Weir.
  2011.
\newblock \href{https://doi.org/10.1162/COLI\_a\_00060}{Dependency parsing
  schemata and mildly non-projective dependency parsing}.
\newblock {\em Computational Linguistics\/} 37(3):541--586.
\newblock
  \href{https://doi.org/10.1162/COLI\_a\_00060}{https://doi.org/10.1162/COLI\_a\_00060}.

\bibitem[{Gr\"{a}del et~al.(2005)Gr\"{a}del, Kolaitis, Libkin, Marx, Spencer,
  Vardi, Venema, and Weinstein}]{Gradel:2005:FMT:1206819}
Erich Gr\"{a}del, P.~G. Kolaitis, L.~Libkin, M.~Marx, J.~Spencer, Moshe~Y.
  Vardi, Y.~Venema, and Scott Weinstein. 2005.
\newblock {\em Finite Model Theory and Its Applications (Texts in Theoretical
  Computer Science. An EATCS Series)\/}.
\newblock Springer-Verlag New York, Inc., Secaucus, NJ, USA.

\bibitem[{Greibach(1973)}]{Greibach}
Sheila Greibach. 1973.
\newblock \href{https://doi.org/10.1137/0202025}{The hardest context-free
  language}.
\newblock {\em SIAM Journal on Computing\/} 2(4):304--310.
\newblock
  \href{https://doi.org/10.1137/0202025}{https://doi.org/10.1137/0202025}.

\bibitem[{Guruswami et~al.(2011)Guruswami, H\aa{}stad, Manokaran, Raghavendra,
  and Charikar}]{Guruswami-etal-2011}
Venkatesan Guruswami, Johan H\aa{}stad, Rajsekar Manokaran, Prasad Raghavendra,
  and Moses Charikar. 2011.
\newblock \href{https://doi.org/10.1137/090756144}{Beating the random ordering
  is hard: Every ordering {CSP} is approximation resistant}.
\newblock {\em SIAM Journal on Computing\/} 40(3):878–914.
\newblock
  \href{https://doi.org/10.1137/090756144}{https://doi.org/10.1137/090756144}.

\bibitem[{Hulden(2009)}]{FOMA}
Mans Hulden. 2009.
\newblock \href{http://www.aclweb.org/anthology/E09-2008}{{Foma}: a
  finite-state compiler and library}.
\newblock In {\em Proceedings of the Demonstrations Session at EACL 2009\/}.
  Association for Computational Linguistics, Athens, Greece, pages 29--32.
\newblock
  \href{http://www.aclweb.org/anthology/E09-2008}{http://www.aclweb.org/anthology/E09-2008}.

\bibitem[{Hulden(2011)}]{Hulden2011}
Mans Hulden. 2011.
\newblock \href{https://doi.org/10.1007/978-3-642-20095-3\_14}{Parsing {CFG}s
  and {PCFG}s with a {C}homsky-{S}ch{\"u}tzenberger representation}.
\newblock In Zygmunt Vetulani, editor, {\em Human Language Technology.
  Challenges for Computer Science and Linguistics: 4th Language and Technology
  Conference, LTC 2009, Poznan, Poland, November 6-8, 2009, Revised Selected
  Papers\/}, Springer Berlin Heidelberg, Berlin, Heidelberg, pages 151--160.
\newblock
  \href{https://doi.org/10.1007/978-3-642-20095-3\_14}{https://doi.org/10.1007/978-3-642-20095-3\_14}.

\bibitem[{Kaplan and Kay(1994)}]{Kaplan:1994:RMP:204915.204917}
Ronald~M. Kaplan and Martin Kay. 1994.
\newblock \href{http://dl.acm.org/citation.cfm?id=204915.204917}{Regular models
  of phonological rule systems}.
\newblock {\em Computational Linguistics\/} 20(3):331--378.
\newblock
  \href{http://dl.acm.org/citation.cfm?id=204915.204917}{http://dl.acm.org/citation.cfm?id=204915.204917}.

\bibitem[{Kuhlmann(2015)}]{Kuhlmann2015}
Marco Kuhlmann. 2015.
\newblock \href{https://arxiv.org/abs/1504.04993}{Tabulation of noncrossing
  acyclic digraphs}.
\newblock arXiv:1504.04993.
\newblock
  \href{https://arxiv.org/abs/1504.04993}{https://arxiv.org/abs/1504.04993}.

\bibitem[{Kuhlmann and Johnsson(2015)}]{Kuhlmann-Jonsson-2015}
Marco Kuhlmann and Peter Johnsson. 2015.
\newblock \href{http://aclweb.org/anthology/Q/Q15/Q15-1040.pdf}{Parsing to
  noncrossing dependency graphs}.
\newblock {\em Transactions of the Association for Computational Linguistics\/}
  3:559--570.
\newblock
  \href{http://aclweb.org/anthology/Q/Q15/Q15-1040.pdf}{http://aclweb.org/anthology/Q/Q15/Q15-1040.pdf}.

\bibitem[{Kuhlmann and Oepen(2016)}]{KuhOe2016}
Marco Kuhlmann and Stephan Oepen. 2016.
\newblock \href{https://doi.org/10.1162/COLI\_a\_00268}{Towards a catalogue of
  linguistic graph banks}.
\newblock {\em Computational Linguistics\/} 42(4):819--827.
\newblock
  \href{https://doi.org/10.1162/COLI\_a\_00268}{https://doi.org/10.1162/COLI\_a\_00268}.

\bibitem[{Lang(1994)}]{LA92}
Bernard Lang. 1994.
\newblock
  \href{http://onlinelibrary.wiley.com/doi/10.1111/j.1467-8640.1994.tb00011.x/full}{Recognition
  can be harder than parsing}.
\newblock {\em Computational Intelligence\/} 10(4):486--494.
\newblock
  \href{http://onlinelibrary.wiley.com/doi/10.1111/j.1467-8640.1994.tb00011.x/full}{http://onlinelibrary.wiley.com/doi/10.1111/j.1467-8640.1994.tb00011.x/full}.

\bibitem[{Lange(1997)}]{Lange-1997}
Klaus-J\"orn Lange. 1997.
\newblock \href{http://dl.acm.org/citation.cfm?id=695352}{An unambiguous class
  possessing a complete set}.
\newblock In Morvan Reischuk, editor, {\em STACKS'97 Proceedings\/}. Springer,
  volume 1200 of {\em Lecture Notes in Computer Science\/}.
\newblock
  \href{http://dl.acm.org/citation.cfm?id=695352}{http://dl.acm.org/citation.cfm?id=695352}.

\bibitem[{Marcus(1967)}]{Marcus1967}
S.~Marcus. 1967.
\newblock {\em Algebraic Linguistics; Analytical Models\/}, volume~29 of {\em
  Mathematics in Science and Engineering\/}.
\newblock Academic Press, New York and London.

\bibitem[{McDonald et~al.(2005)McDonald, Pereira, Ribarov, and
  Hajic}]{McDonald}
Ryan McDonald, Fernando Pereira, Kiril Ribarov, and Jan Hajic. 2005.
\newblock
  \href{http://www.aclweb.org/anthology/H/H05/H05-1066.pdf}{Non-projective
  dependency parsing using spanning tree algorithms}.
\newblock In {\em Proceedings of Human Language Technology Conference and
  Conference on Empirical Methods in Natural Language Processing\/}.
  Association for Computational Linguistics, Vancouver, British Columbia,
  Canada, pages 523--530.
\newblock
  \href{http://www.aclweb.org/anthology/H/H05/H05-1066.pdf}{http://www.aclweb.org/anthology/H/H05/H05-1066.pdf}.

\bibitem[{Nederhof and Satta(2003)}]{Nederhof-2003}
Mark-Jan Nederhof and Giorgio Satta. 2003.
\newblock Probabilistic parsing as intersection.
\newblock In {\em 8th International Workshop on Parsing Technologies\/}. LORIA,
  Nancy, France, pages 137--148.

\bibitem[{{OEIS Foundation Inc.}(2017)}]{OEIS}
{OEIS Foundation Inc.} 2017.
\newblock The on-line encyclopedia of integer sequences.
\newblock \url{http://oeis.org}, read on 15 January 2017.

\bibitem[{Oepen et~al.(2015)Oepen, Kuhlmann, Miyao, Zeman, Cinkova, Flickinger,
  Hajic, and Uresova}]{oepen-EtAl:2015:SemEval}
Stephan Oepen, Marco Kuhlmann, Yusuke Miyao, Daniel Zeman, Silvie Cinkova, Dan
  Flickinger, Jan Hajic, and Zdenka Uresova. 2015.
\newblock \href{http://www.aclweb.org/anthology/S15-2153}{Semeval 2015 task 18:
  Broad-coverage semantic dependency parsing}.
\newblock In {\em Proceedings of the 9th International Workshop on Semantic
  Evaluation (SemEval 2015)\/}. Association for Computational Linguistics,
  Denver, Colorado, pages 915--926.
\newblock
  \href{http://www.aclweb.org/anthology/S15-2153}{http://www.aclweb.org/anthology/S15-2153}.

\bibitem[{Oflazer(2003)}]{DBLP:journals/coling/Oflazer03}
Kemal Oflazer. 2003.
\newblock \href{https://doi.org/10.1162/089120103322753338}{Dependency parsing
  with an extended finite-state approach}.
\newblock {\em Computational Linguistics\/} 29(4):515--544.
\newblock
  \href{https://doi.org/10.1162/089120103322753338}{https://doi.org/10.1162/089120103322753338}.

\bibitem[{Pitler et~al.(2013)Pitler, Kannan, and Marcus}]{Pitler2013}
Emily Pitler, Sampath Kannan, and Mitchell Marcus. 2013.
\newblock \href{http://aclweb.org/anthology/Q13-1002}{Finding optimal
  1-endpoint-crossing trees}.
\newblock {\em Transactions of the Association for Computational Linguistics\/}
  1:13--24.
\newblock
  \href{http://aclweb.org/anthology/Q13-1002}{http://aclweb.org/anthology/Q13-1002}.

\bibitem[{Rebane and Pearl(1987)}]{Rebane-Pearl-1987}
George Rebane and Judea Pearl. 1987.
\newblock \href{http://dl.acm.org/citation.cfm?id=3023784}{The recovery of
  causal poly-trees from statistical data}.
\newblock In {\em Proceedings of the 3rd Annual Conference on Uncertainty in
  Artificial Intelligence (UAI 1987)\/}. Seattle, WA, pages 222--228.
\newblock
  \href{http://dl.acm.org/citation.cfm?id=3023784}{http://dl.acm.org/citation.cfm?id=3023784}.

\bibitem[{Roche(1996)}]{Roche:1996:TPF:974697.974706}
Emmanuel Roche. 1996.
\newblock \href{https://doi.org/10.1017/S1351324997001605}{Transducer parsing
  of free and frozen sentences}.
\newblock {\em Natural Language Engineering\/} 2(4):345--350.
\newblock
  \href{https://doi.org/10.1017/S1351324997001605}{https://doi.org/10.1017/S1351324997001605}.

\bibitem[{Schluter(2014)}]{Schluter2014}
Natalie Schluter. 2014.
\newblock \href{http://www.aclweb.org/anthology/W/W14/W14-2412.pdf}{On maximum
  spanning {DAG} algorithms for semantic {DAG} parsing}.
\newblock In {\em Proceedings of the ACL 2014 Workshop on Semantic Parsing\/}.
  Association for Computational Linguistics, Baltimore, MD, pages 61--65.
\newblock
  \href{http://www.aclweb.org/anthology/W/W14/W14-2412.pdf}{http://www.aclweb.org/anthology/W/W14/W14-2412.pdf}.

\bibitem[{Schluter(2015)}]{Schluter2015}
Natalie Schluter. 2015.
\newblock \href{http://www.aclweb.org/anthology/S15-1031}{The complexity of
  finding the maximum spanning {DAG} and other restrictions for {DAG} parsing
  of natural language}.
\newblock In {\em Proceedings of the Fourth Joint Conference on Lexical and
  Computational Semantics\/}. Association for Computational Linguistics,
  Denver, Colorado, pages 259--268.
\newblock
  \href{http://www.aclweb.org/anthology/S15-1031}{http://www.aclweb.org/anthology/S15-1031}.

\bibitem[{Yli-Jyr\"a(2004)}]{YJ2004COLING}
Anssi Yli-Jyr\"a. 2004.
\newblock
  \href{https://www.aclweb.org/anthology/W/W04/W04-1504.pdf}{Axiomatization of
  restricted non-projective dependency trees through finite-state constraints
  that analyse crossing bracketings}.
\newblock In Geert-Jan~M. Kruijff and Denys Duchier, editors, {\em COLING 2004
  Recent Advances in Dependency Grammar\/}. COLING, Geneva, Switzerland, pages
  25--32.
\newblock
  \href{https://www.aclweb.org/anthology/W/W04/W04-1504.pdf}{https://www.aclweb.org/anthology/W/W04/W04-1504.pdf}.

\bibitem[{Yli{-}Jyr{\"{a}}(2005)}]{DBLP:journals/ijfcs/Yli-Jyra05}
Anssi Yli{-}Jyr{\"{a}}. 2005.
\newblock \href{https://doi.org/10.1142/S0129054105003169}{Approximating
  dependency grammars through intersection of star-free regular languages}.
\newblock {\em Int. J. Found. Comput. Sci.\/} 16(3):565--579.
\newblock
  \href{https://doi.org/10.1142/S0129054105003169}{https://doi.org/10.1142/S0129054105003169}.

\bibitem[{Yli{-}Jyr{\"{a}}(2012)}]{DBLP:conf/birthday/Yli-Jyra12}
Anssi Yli{-}Jyr{\"{a}}. 2012.
\newblock \href{https://doi.org/10.1007/978-3-642-30773-7\_10}{On dependency
  analysis via contractions and weighted {FST}s}.
\newblock In Diana Santos, Krister Lind{\'{e}}n, and Wanjiku Ng'ang'a, editors,
  {\em Shall We Play the Festschrift Game?, Essays on the Occasion of Lauri
  Carlson's 60th Birthday\/}. Springer, pages 133--158.
\newblock
  \href{https://doi.org/10.1007/978-3-642-30773-7\_10}{https://doi.org/10.1007/978-3-642-30773-7\_10}.

\bibitem[{Yli-Jyr\"a et~al.(2012)Yli-Jyr\"a, Piitulainen, and
  Voutilainen}]{YJ2012FSMNLP}
Anssi Yli-Jyr\"a, Jussi Piitulainen, and Atro Voutilainen. 2012.
\newblock \href{http://www.aclweb.org/anthology/W12-6218}{Refining the design
  of a contracting finite-state dependency parser}.
\newblock In I\~naki Alegria and Mans Hulden, editors, {\em Proceedings of the
  10th International Workshop on Finite State Methods and Natural Language
  Processing\/}. Association for Computational Linguistics, Donostia--San
  Sebasti\'an, Spain, pages 108--115.
\newblock
  \href{http://www.aclweb.org/anthology/W12-6218}{http://www.aclweb.org/anthology/W12-6218}.

\bibitem[{Zhang and Nivre(2011)}]{Zhang2011rich}
Yue Zhang and Joakim Nivre. 2011.
\newblock \href{http://www.aclweb.org/anthology/P11-2033}{Transition-based
  dependency parsing with rich non-local features}.
\newblock In {\em Proceedings of the 49th Annual Meeting of the Association for
  Computational Linguistics: Human Language Technologies\/}. Association for
  Computational Linguistics, Portland, Oregon, USA, pages 188--193.
\newblock
  \href{http://www.aclweb.org/anthology/P11-2033}{http://www.aclweb.org/anthology/P11-2033}.

\end{thebibliography}
\end{document}